%% file: main_arxiv.tex
\documentclass{article}

\PassOptionsToPackage{numbers, compress}{natbib}
\usepackage[final]{neurips_2025}

\usepackage[utf8]{inputenc} % allow utf-8 input
\usepackage[T1]{fontenc}    % use 8-bit T1 fonts
\usepackage{hyperref}       % hyperlinks
\usepackage{url}            % simple URL typesetting
\usepackage{booktabs}       % professional-quality tables
\usepackage{nicefrac}       % compact symbols for 1/2, etc.
\usepackage{microtype}      % microtypography
\usepackage{xcolor}         % colors
\usepackage{amsmath,amsthm,amsfonts,amssymb,amscd}
\usepackage[shortlabels]{enumitem}
\usepackage{bm,bbm}
\usepackage{algpseudocode, algorithm}
\usepackage{pgfplots}
\usetikzlibrary{math,external}
\tikzsetexternalprefix{figs/}
\definecolor{forestgreen}{rgb}{0,0.5,0}

\usepackage{defns_penalty_bai}

\title{\Large Balancing Performance and Costs in Best Arm Identification}

\author{%
	Michael O.~Harding \\
	Department of Statistics\\
	University of Wisconsin-Madison\\
	\texttt{moharding@wisc.edu} \\
	\And
	Kirthevasan Kandasamy \\
	Department of Computer Science\\
	University of Wisconsin-Madison\\
	\texttt{kandasamy@cs.wisc.edu} \\
}

\begin{document}

\maketitle

\input{tables_figures.tex}

\input{abstract}

\input{intro}

\input{two_arm}

\input{k_arm}

\input{experiments}
\input{conclusion}
\ack
\input{ack}

\bibliography{bib_penalty_bai}

\appendix
\input{app_prior_results}

\input{app_lb_proof}

\input{app_oracle_ub}

\input{app_alg_ub}

\input{app_experiments}

%%%%%%%%%%%%%%%%%%%%%%%%%%%%%%%%%%%%%%%%%%%%%%%%%%%%%%%%%%%%

\end{document}

%% file: tables_figures.tex
\newcommand{\insertAlgoMain}{%
    \begin{algorithm}[!ht]
        \caption{\textbf{D}ynamically \textbf{B}udgeted \textbf{C}ost-\textbf{A}dapted \textbf{R}isk-minimizing \textbf{E}limination}
        \label{alg:our-alg}
        \begin{algorithmic}[1]
            \Require{%sub-Gaussian parameter $\sigma\,$, number of arms $K\,$
        Dynamic budget function $\Nstar\,$, Confidence $\delta$}
            \State \emph{Initialization}: $\muhat_k(0) = 0\;\forall\;k\in[K]\,, e_0=0\,, t=0\,, n=0\,, S = [K]$
            \While{$n\leq \Nstar(\abs{S})$ AND $\abs{S}>1$}
            \State $n\gets n+1$
            \For{$k\in S$}
            \State $t\gets t+1$
            \State $A_t\gets k\,$, Observe $X_t\sim\nu_{A_t}$
            \EndFor
            \State $\muhat_k(n)\gets \frac{1}{n}\sum_{s=1}^{t}\1_{\{k\}}(A_s)X_s\,,$ for $k\in S$ 
            \State $e_n\gets\sqrt{4\sigma^2n^{-1}\log(Kn\delta^{-1})}$.
            \State $S\gets S\setminus\left\{k\in S: \max\limits_{\ell\in
S}\muhat_\ell(n)-\muhat_k(n) > e_n\right\}$.
            \EndWhile\\
            \Return $\argmax\limits_{a\in S}\muhat_a(n)$ (breaking ties randomly)
        \end{algorithmic}
    \end{algorithm}
}

\newcommand{\insertTwoArmBounds}{%
\begin{figure}[t]
	\centering
	\includegraphics{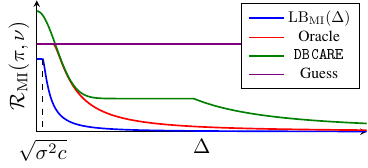}
	\hspace{0.05\linewidth}
	\includegraphics{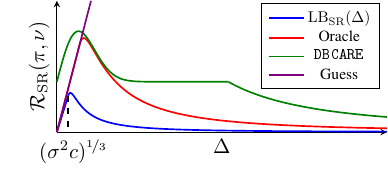}
	\caption{Illustrations of the lower and upper bounds on the risk for $\riskone$ (on the left) and $\risktwo$ (on the right) in the 2-arm case presented throughout \S~\ref{sec:2-arm}, with the performance of the policy which guesses an arm at random without pulling at all (Guess) included as a point of reference.}
	\label{fig:boundillus}
\end{figure}
}

\newcommand{\insertMainSimulation}{%
	\begin{figure}[t]
		\hspace*{\fill}
		\begin{minipage}[t]{0.2435\linewidth}
			\centering
			\vspace{0pt}
			\includegraphics[width=0.95\linewidth]{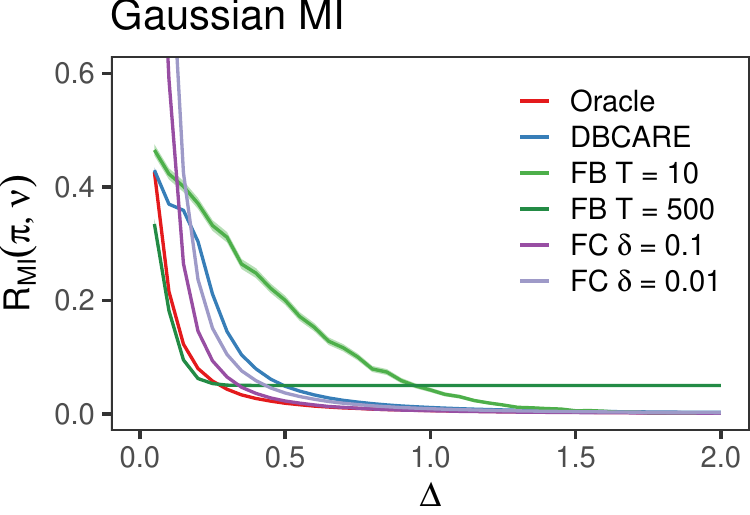}
		\end{minipage}
		\hfill
		\begin{minipage}[t]{0.2435\linewidth}
			\centering
			\vspace{0pt}
			\includegraphics[width=0.95\linewidth]{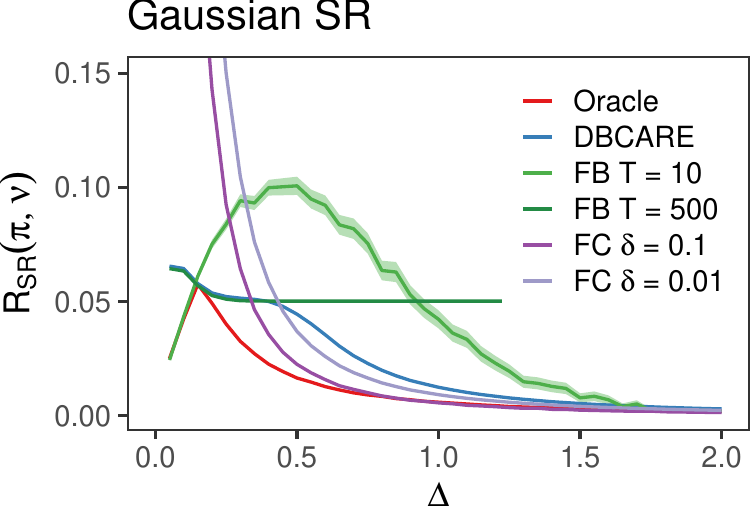}
		\end{minipage}
		\hfill
%		\hspace*{\fill}
%		
%		\hspace*{\fill}
		\begin{minipage}[t]{0.2435\linewidth}
			\centering
			\vspace{0pt}
			\includegraphics[width=0.95\linewidth]{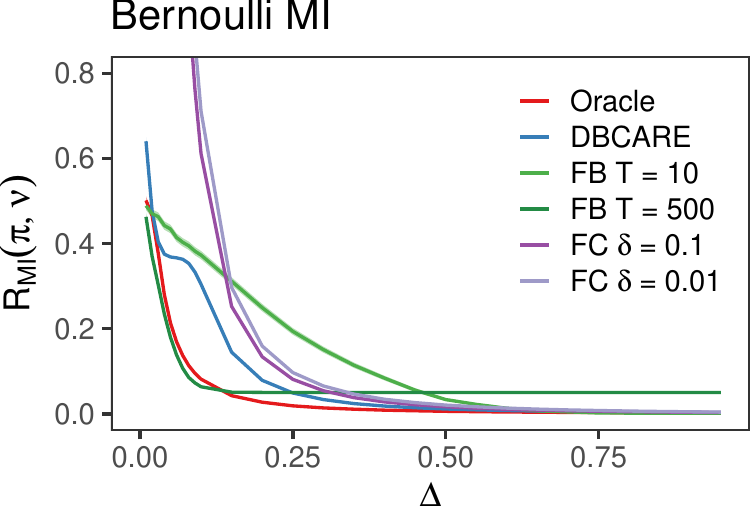}
		\end{minipage}
		\hfill
		\begin{minipage}[t]{0.2435\linewidth}
			\centering
			\vspace{0pt}
			\includegraphics[width=0.95\linewidth]{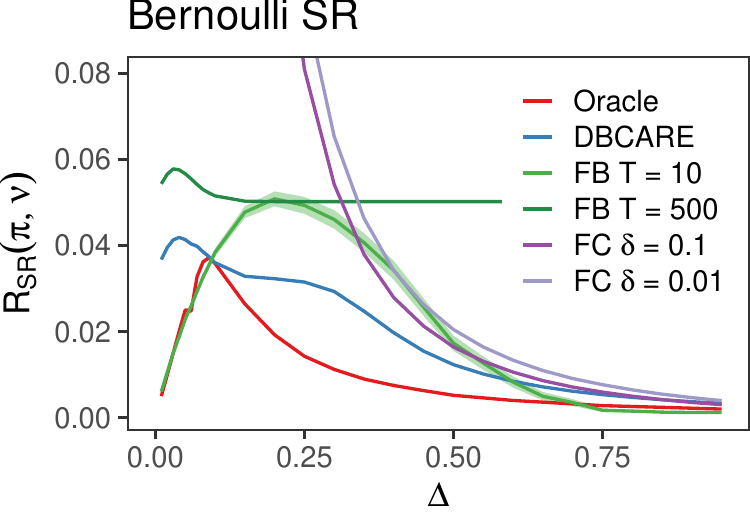}
		\end{minipage}
		\hspace*{\fill}
		\caption{Comparisons between the oracular policy, \algname, and fixed budget and confidence algorithms for $\riskone$ and $\risktwo$. $Y$-axes are adjusted per setting to highlight problem-specific behavior. Confidence regions represent empirical average risk $\pm$ 2 SE.}
		\label{fig:exp-two-arm}
	\end{figure}
}

\newcommand{\insertKSparse}{%
	\begin{figure}[t]
		\hspace*{\fill}
		\begin{minipage}[t]{0.32\linewidth}
			\centering
			\vspace{0pt}
			\includegraphics[width=0.95\linewidth]{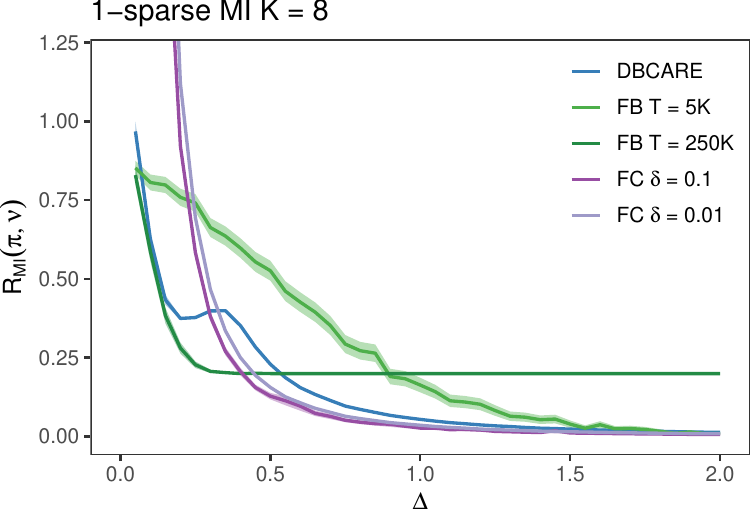}
		\end{minipage}
		\hfill
		\begin{minipage}[t]{0.32\linewidth}
			\centering
			\vspace{0pt}
			\includegraphics[width=0.95\linewidth]{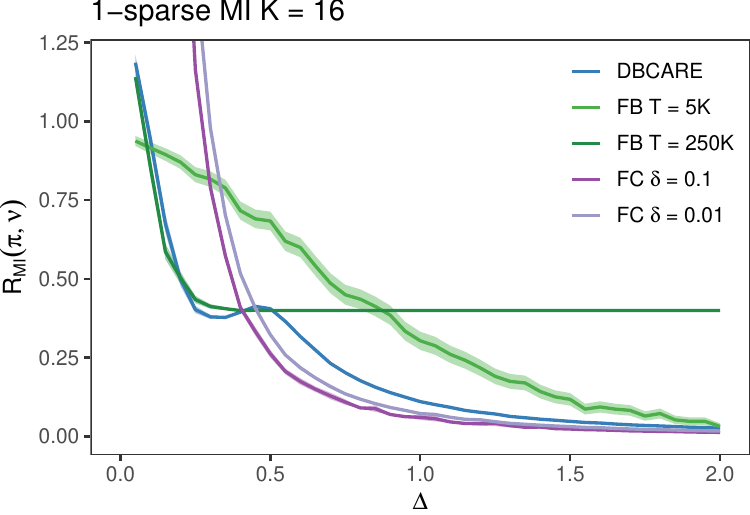}
		\end{minipage}
		\hfill
		\begin{minipage}[t]{0.32\linewidth}
			\centering
			\vspace{0pt}
			\includegraphics[width=0.95\linewidth]{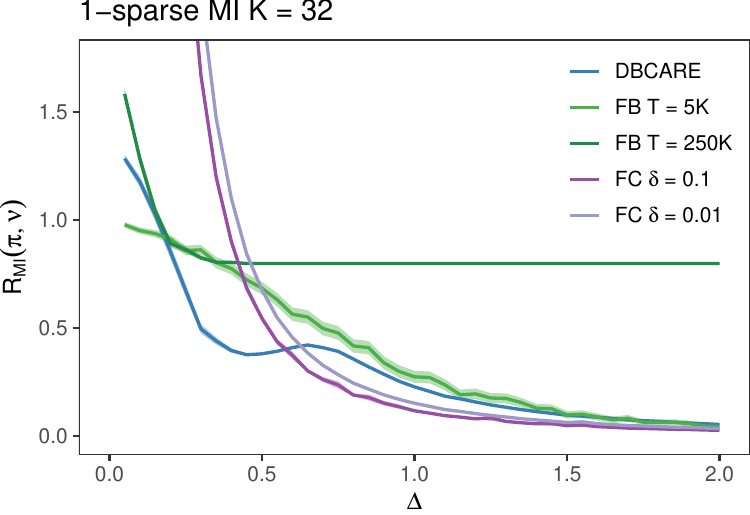}
		\end{minipage}
		\hspace*{\fill}
		
		\hspace*{\fill}
		\begin{minipage}[t]{0.32\linewidth}
			\centering
			\vspace{0pt}
			\includegraphics[width=0.95\linewidth]{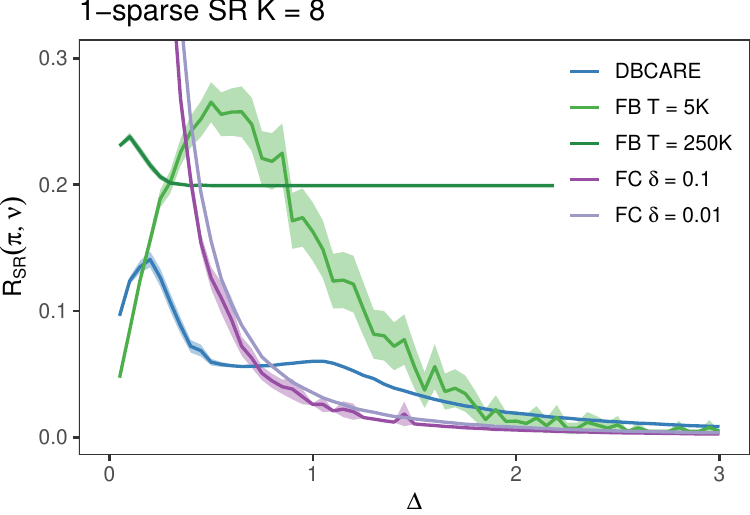}
		\end{minipage}
		\hfill
		\begin{minipage}[t]{0.32\linewidth}
			\centering
			\vspace{0pt}
			\includegraphics[width=0.95\linewidth]{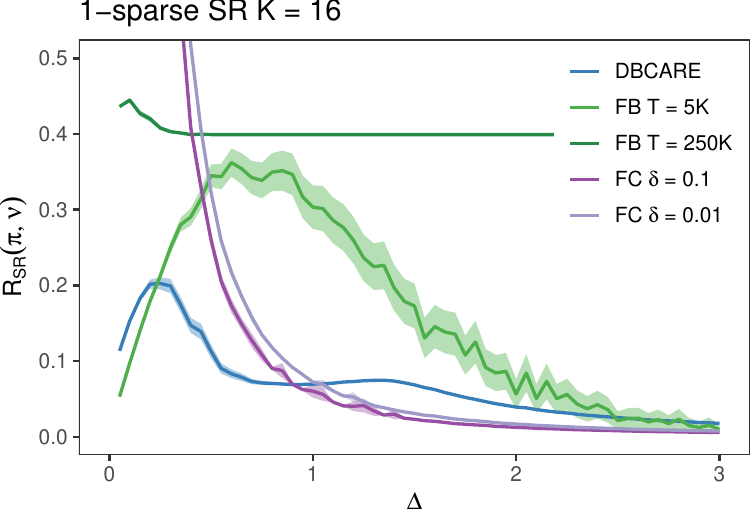}
		\end{minipage}
		\hfill
		\begin{minipage}[t]{0.32\linewidth}
			\centering
			\vspace{0pt}
			\includegraphics[width=0.95\linewidth]{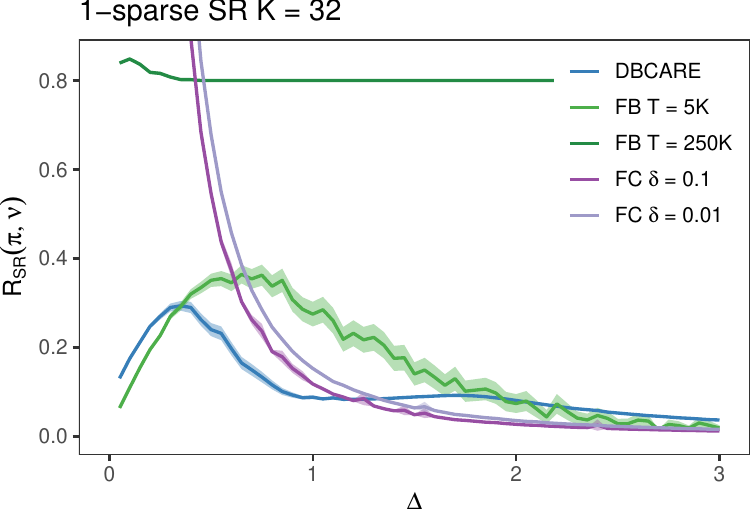}
		\end{minipage}
		\hspace*{\fill}
		\caption{Comparisons between \algname{} and fixed budget and confidence algorithms for $\riskone$ and $\risktwo$ in the $K$-arm 1-sparse setting. $Y$-axes are adjusted per setting to highlight problem-specific behavior. Confidence regions represent empirical average risk $\pm$2 SE.}
		\label{fig:exp-k-sparse}
	\end{figure}
}

\newcommand{\insertKLinear}{%
	\begin{figure}[t]
		\hspace*{\fill}
		\begin{minipage}[t]{0.32\linewidth}
			\centering
			\vspace{0pt}
			\includegraphics[width=0.95\linewidth]{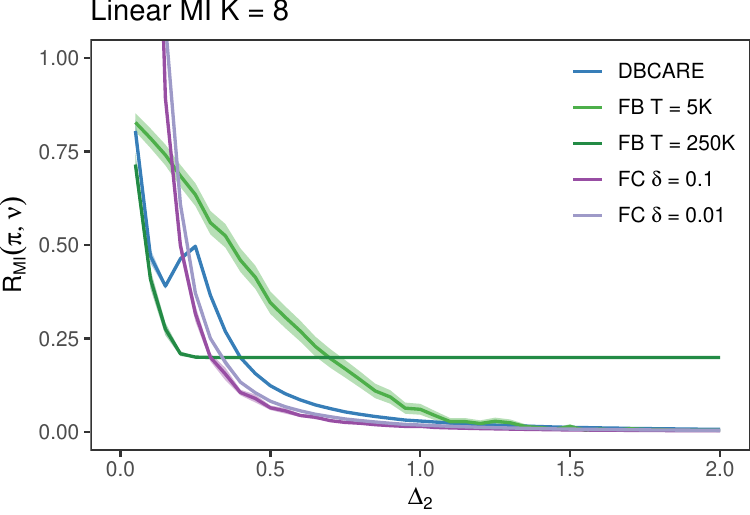}
		\end{minipage}
		\hfill
		\begin{minipage}[t]{0.32\linewidth}
			\centering
			\vspace{0pt}
			\includegraphics[width=0.95\linewidth]{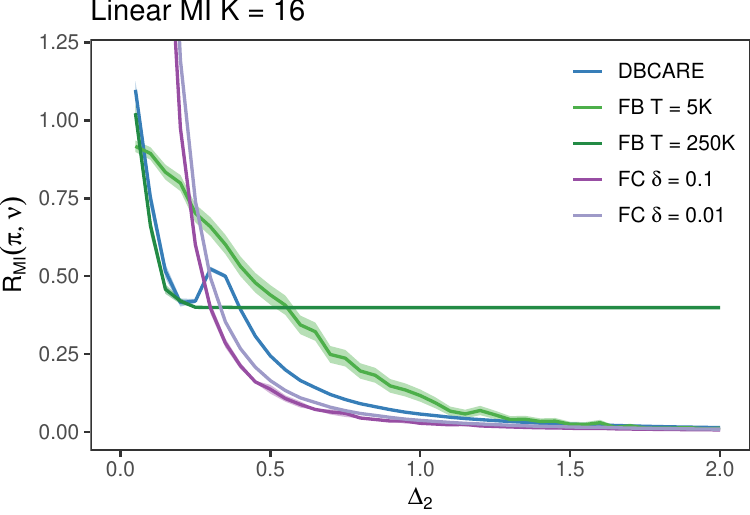}
		\end{minipage}
		\hfill
		\begin{minipage}[t]{0.32\linewidth}
			\centering
			\vspace{0pt}
			\includegraphics[width=0.95\linewidth]{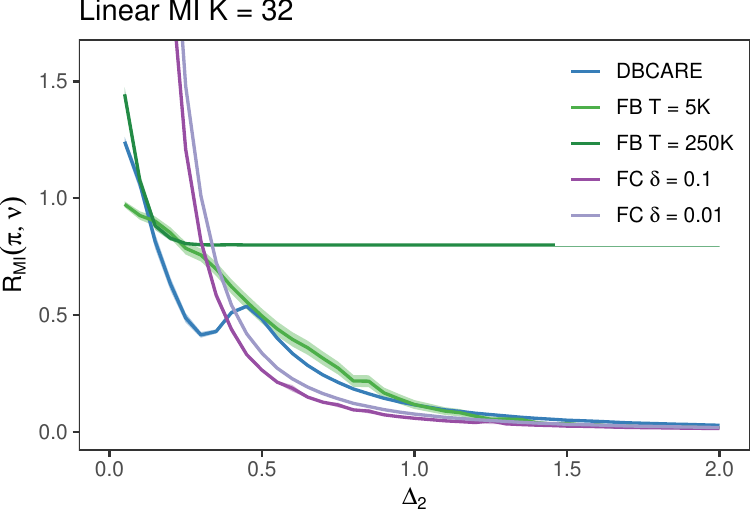}
		\end{minipage}
		\hspace*{\fill}
		
		\hspace*{\fill}
		\begin{minipage}[t]{0.32\linewidth}
			\centering
			\vspace{0pt}
			\includegraphics[width=0.95\linewidth]{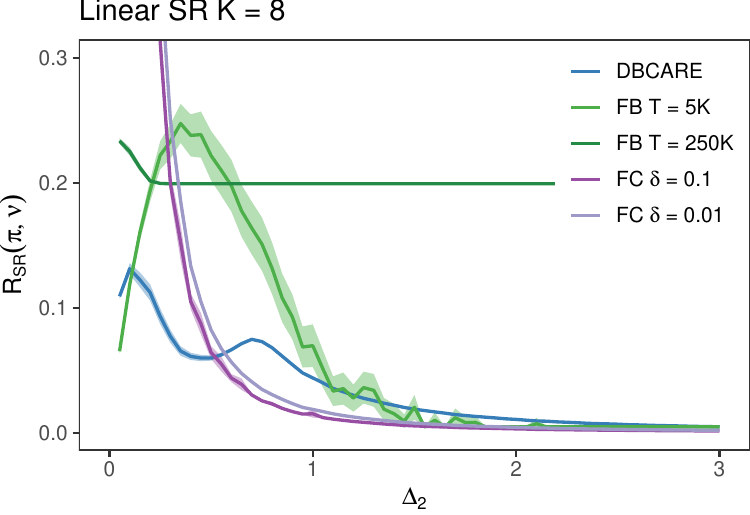}
		\end{minipage}
		\hfill
		\begin{minipage}[t]{0.32\linewidth}
			\centering
			\vspace{0pt}
			\includegraphics[width=0.95\linewidth]{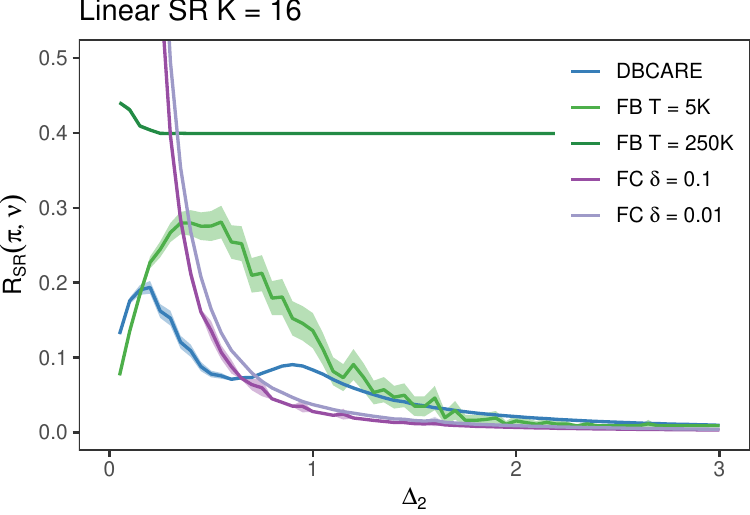}
		\end{minipage}
		\hfill
		\begin{minipage}[t]{0.32\linewidth}
			\centering
			\vspace{0pt}
			\includegraphics[width=0.95\linewidth]{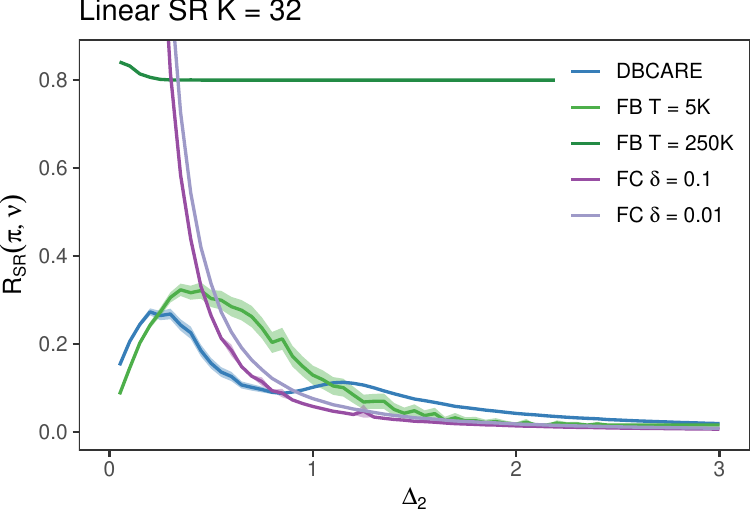}
		\end{minipage}
		\hspace*{\fill}
		\caption{Comparisons between \algname{} and fixed budget and confidence algorithms for $\riskone$ and $\risktwo$ in the $K$-arm linear decay setting. $Y$-axes are adjusted per setting to highlight problem-specific behavior. Confidence regions represent empirical average risk $\pm$2 SE.}
		\label{fig:exp-k-linear}
	\end{figure}
}

\newcommand{\insertDrugDiscovery}{%
	\begin{figure}[t]
		\hspace*{\fill}
		\begin{minipage}[t]{0.449\linewidth}
			\centering
			\vspace{0pt}
			\includegraphics[width=0.95\linewidth]{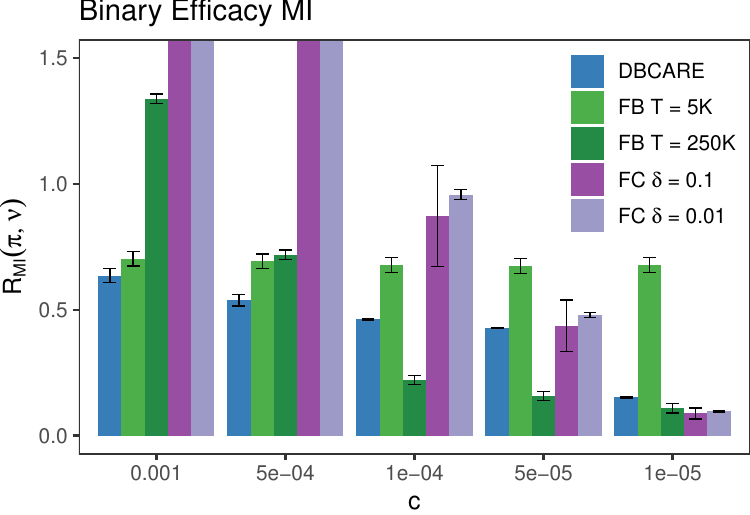}
		\end{minipage}
		\hfill
		\begin{minipage}[t]{0.449\linewidth}
			\centering
			\vspace{0pt}
			\includegraphics[width=0.95\linewidth]{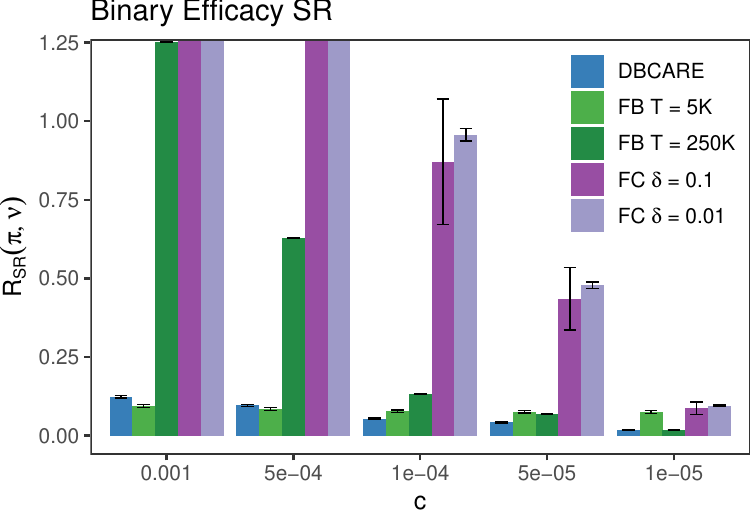}
		\end{minipage}
		%\hfill
		\hspace*{\fill}
		
		\hspace*{\fill}
		\begin{minipage}[t]{0.449\linewidth}
			\centering
			\vspace{0pt}
			\includegraphics[width=0.95\linewidth]{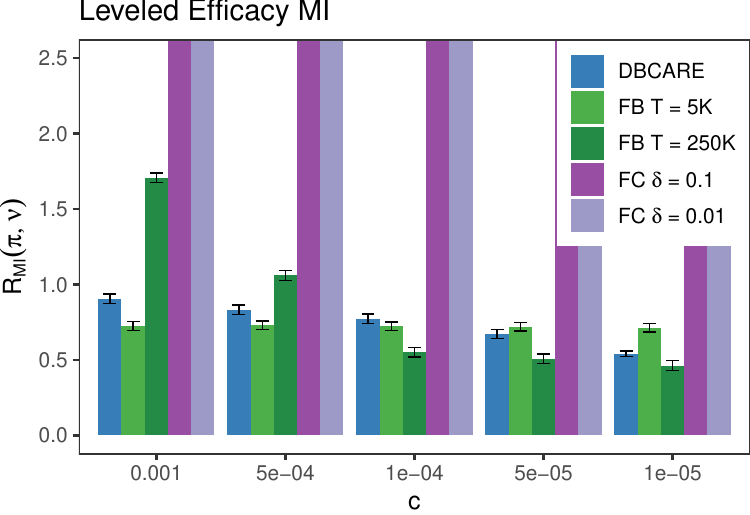}
		\end{minipage}
		\hfill
		\begin{minipage}[t]{0.449\linewidth}
			\centering
			\vspace{0pt}
			\includegraphics[width=0.95\linewidth]{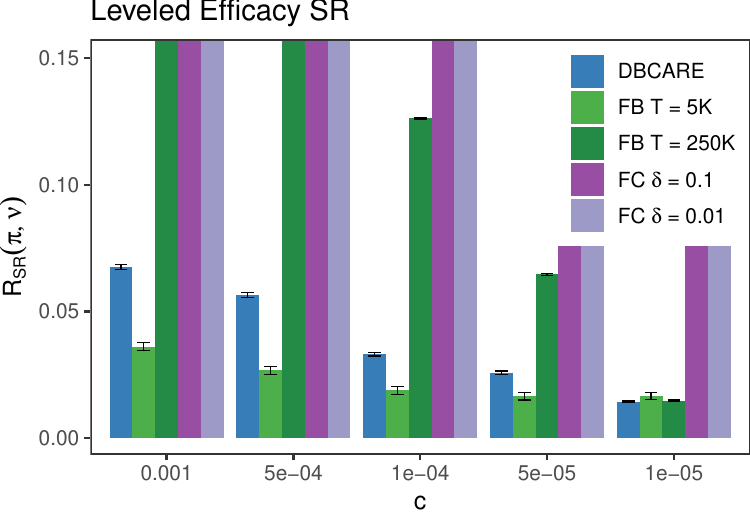}
		\end{minipage}
		\hspace*{\fill}
		\caption{Comparisons between \algname{} and fixed budget and confidence algorithms for $\riskone$ and $\risktwo$ on a drug discovery dataset. $Y$-axes are adjusted per setting to highlight problem-specific behavior. Error bars represent empirical average risk $\pm$2 SE.}
		\label{fig:exp-drug}
	\end{figure}
}

%% file: abstract.tex
	\begin{abstract}
    We consider the problem of identifying the best arm in a multi-armed bandit model.
Despite a wealth of literature in the traditional fixed budget and fixed confidence
regimes of the best arm identification problem, it still remains a mystery to most
practitioners as to how to choose an approach and corresponding budget or confidence
parameter. We propose a new formalism to avoid this dilemma altogether by minimizing a
risk functional which explicitly balances the performance of the recommended arm and the
cost incurred by learning this arm. In this
framework, a cost is incurred for each observation during the sampling phase, and upon
recommending an arm, a performance penalty is incurred for identifying a suboptimal arm.
The learner's goal is to minimize the sum of the penalty and cost. This new regime
mirrors the priorities of many practitioners, \eg maximizing profit in an A/B
testing framework, better than classical fixed budget or confidence
settings. We derive theoretical lower bounds for the risk of each of two
choices for the performance penalty, the probability of misidentification and the simple
regret, and propose an algorithm called \algname{} to match
these lower bounds up to polylog factors on nearly all problem instances. We then
demonstrate the performance of \algname{} on a number of simulated models, comparing to
fixed budget and confidence algorithms to show the shortfalls of existing BAI paradigms on this problem.
	\end{abstract}

%% file: intro.tex
	\section{Introduction}

Best Arm Identification (BAI) in multi-armed bandits is a fundamental problem in
decision-making under uncertainty. The objective is to identify the arm with the highest
expected reward by adaptively drawing samples from distributions associated with each
arm.
BAI arises in many real-world applications.
In advertising, arms represent different ads, and the aim is to find the ad which
maximizes revenue generated~\citep{geng2021comparison}.
In statistical model selection, arms represent different
hyperparameter configurations, and the aim is to find the best-performing one with minimal
computational resources~\cite{Jamieson2016NonstochasticBest}.
Traditionally, BAI has been studied under two paradigms: the \emph{fixed budget}
setting~\citep{Audibert2010,Bubeck2011PureExploration},
which seeks to maximize performance---\ie the ability of a policy to recover the optimal arm---within
a given sampling budget, and the
\emph{fixed performance}  (e.g., fixed confidence~\citep{Mannor_Tsitsiklis2004,EvenDar2006}) setting,
which aims to minimize the number of samples needed to meet a target performance level.
While algorithms for these settings have been successfully deployed in many
real-world
settings~\citep{li2018hyperband,misra2021rubberband,zhang2024humor,geng2021comparison},
these settings are not a natural fit for all use cases.
For instance, while determining the best arm is desirable,
a slightly suboptimal choice may be acceptable if the cost of distinguishing
between top candidates is prohibitively high. On the other hand, it is often unnecessary
to
continue sampling until reaching some pre-specified horizon when there is already enough
evidence to determine the optimal arm.

To this end, we propose a novel paradigm for BAI, in which a
policy should explicitly balance performance and sampling cost on the fly,
without being constrained by a
fixed performance level or a pre-specified sampling budget. This framework allows policies
to \textit{adaptively} terminate according to the difficulty of the problem.
The following is an example where such a framework would be natural.

\begin{example}[Advertising]
	Consider a firm choosing among $K$ versions of an ad.
	To inform its choice, the firm may show versions to
	participants in a focus group (arm pull), incurring a cost $c$ per showing.
	The firm wishes to choose an algorithm to maximize the expected profit, \ie
	the expected revenue of the selected ad ($\Ihat$) minus the expected cost of
	the sampling procedure:  
$ \E[\text{revenue}_{\Ihat}] - c\E[\#\text{ arm pulls}]$.
Letting $\Istar$ be the ad with the highest expected revenue,
then maximizing expected profit can be equivalently stated as minimizing
$\E[\text{revenue}_{\Istar}-\text{revenue}_{\Ihat}] + c\E[\#\text{ arm pulls}]$.
Traditional fixed budget or confidence algorithms would be a poor fit for this
problem, as it is unclear how one should choose the budget or confidence level to
optimize the objective.
\end{example}

\subsection{Model}
\label{sec:model}
We will now formally introduce our setting.
A learner has access to a
MAB model $\nu = \{\nu_a\}_{a\in[K]}$, which consists of $K$ arms, each associated with a
probability distribution $\nu_a$.
Let $\mu_a = \mathbb{E}_{\nu_a}[X]$ denote the expected reward of arm $a\,.$ Following common conventions in the BAI literature,
we assume without loss of generality that the arms are ordered so that $\muone\geq\mutwo\geq\cdots\geq\muK$
(the learner is unaware of this ordering).
We will assume that for each arm $a\in[K]$,
the distribution $\nu_a$ is
$\sigma$-sub-Gaussian and that $\mu_a \in [0, \rewrange]$.
The learner is aware of $\sigma$ and $B$.

A learner interacts with the bandit model over a sequence of rounds $t=1,2,\dots$.
On round $t$, the learner selects an arm $A_t \in [K]$ according to a policy $\pi$
and observes an independent sample $X_t$ drawn from $\nu_{A_t}$. The choice of $A_t$ may
depend on the history $\{(A_s, X_s)\}_{s=1}^{t-1}$ of previous actions and observations.
Upon termination,
the policy recommends an arm $\Ihat \in [K]$ as the
estimated best arm.

\parahead{Prior work}
Traditionally, BAI has been studied under two main regimes:
\emph{(1) Fixed budget:} The learner is
allowed at most $T \in \mathbb{N}$ samples and aims to minimize either the
\emph{probability of
misidentification}~\citep{Audibert2010} $\Prob(\muone\neq\muIhat)$ or the
\emph{simple regret}~\citep{Bubeck2011PureExploration} $\E[\muone-\muIhat]$,
\ie the expected gap between the
optimal and selected arms. 
\emph{(2) Fixed performance:} The learner must satisfy a specified performance goal
while minimizing the number of samples.
The most common instantiation is
\emph{fixed-confidence} BAI~\citep{Bechhofer11958SequentialMultipleDecision,gabillon2012best}, where
the probability of misidentification $\Prob(\muone\neq\muIhat)$ is at most
a given goal $\delta$.

\parahead{This work}
Both the fixed-budget and fixed-performance formulations neglect practical situations
where
one may not have a pre-specified budget or performance goal,
but would like to trade-off between performance and sampling cost based on problem
difficulty.
Motivated by such considerations, we propose a new setting, where the goal is to
minimize a risk functional that captures both a performance penalty and
the cumulative sampling cost.
Choosing either the probability of
misidentification or the simple regret as the penalty, we study the
following two risk measures:
\begin{align}
\begin{split}
\label{eqn:risk}
\riskone(\pi, \nu) &:=
\Enupi\left[ \indfone\left(\muone \neq \muIhat\right)
    + c\tau \right] =
\Pnupi\left(\muone \neq \muIhat\right) +
c\,\mathbb{E}_{\nu,\pi}[\tau]\,,
\\
\risktwo(\pi, \nu) &:= 
\Enupi\left[\left(\muone - \muIhat\right)
    + c\tau \right] 
=
\mathbb{E}_{\nu,\pi}\left[\muone - \muIhat\right] +
c\,\mathbb{E}_{\nu,\pi}[\tau]\,.
\end{split}
\end{align}
Here, $c > 0$ is the (known) cost required to collect a sample, relative to the performance 
penalty, and 
 $\tau$ is 
the stopping time (total number of samples) of the policy $\pi$.
Moreover, $\Pnupi$ and $\Enupi$ denote the probability and expectation with respect to all randomness arising from the interaction between the policy $\pi$ and the bandit model $\nu$.

\subsection{Summary of our contributions and results}

\parahead{Novel problem formalism}
To the best of our knowledge, we are the first to study this risk-based formalism
for BAI which trades off between performance and sampling costs.
We design policies for both risk measures in~\eqref{eqn:risk},
upper bound the risk, and provide nearly matching lower bounds.

\parahead{Lower bounds}
To summarize our lower bounds,
let $\Delta_k = \muone - \mu_k$ denote the sub-optimality gap of arm $k$,
and let  $H\defeq\sum_{k=2}^K\Delta_k^{-2}$ be a problem complexity parameter~\citep{Mannor_Tsitsiklis2004,EvenDar2006,kalyanakrishnan2012LUCB,jamieson2014lil,LogLog_Karnin2013,Kaufmann2016ComplexityBAI}.
We show that the problem difficulty 
exhibits a phase transition depending on the magnitude of $H$ and the smallest gap $\Delta_2$.
Specifically, in the case of $\riskone$,
when $H\in\bigO((\sigma^2c)^{-1})$, we show that $\riskone \in
\Omega\left(c\sigma^2 H \log\left((c\sigma^2 H)^{-1}\right)\right)$, and otherwise, $\riskone\in\Omega(1)$. In the case of $\risktwo$, when $H\Delta_2^{-1}\in\bigO((\sigma^2c)^{-1})$, we show that
$\risktwo\in \Omega\left(c\sigma^2 H \log\left(\Delta_2(c\sigma^2 H)^{-1}\right)\right)$, and otherwise, $\risktwo\in\Omega(\Delta_2)$.
This phase transition---absent in classical fixed-confidence or fixed-budget
settings---underscores the trade-off between performance and costs inherent to our
setting: probabilistically distinguishing sub-Gaussian arms scales inversely with the size
of the gaps between them, so with small enough gaps it becomes optimal to simply guess the
best arm without incurring the cost of sampling.

\emph{Proof ideas.}
Our proof employs change-of-measure arguments to lower bound the risk associated with any particular algorithm via
an auxiliary function of problem parameters and the expected stopping time of the algorithm, $\Enupi[\tau]$. Crucially, this function is convex in $\Enupi[\tau]$, and minimizing it with respect to $\Enupi[\tau]$ yields lower bounds on the performance of \emph{any} algorithm while additionally revealing the phase transition behavior, via the regions where $\Enupi[\tau]=0$ is optimal.

\parahead{Algorithm}
We propose \algname{} \algdesc{} for this setting. 
\algname{} maintains a subset $S \subset [K]$ of surviving arms and confidence intervals for the
mean values of these arms. 
It takes as input a function $\Nstar:\N\rightarrow \N$ of the size of $S$, which determines the maximum number of times each arm in $S$ may be pulled.
It proceeds in epochs, where in each epoch, every surviving arm is pulled once.
At the end of each epoch, \algname{} eliminates arms that can be confidently identified
as suboptimal based on the confidence intervals.
If any arms are eliminated, the budget for each surviving arm is updated based on
$\Nstar$.
If the budget of arm pulls is exhausted before there is a clear winner, \ie only one
surviving arm, it recommends the surviving arm with the highest empirical mean.
However, if a clear winner emerges before the current budget, it terminates early and
recommends this arm.

\algname{} combines ideas from both fixed-budget and fixed-confidence algorithms for
BAI.
However, unlike fixed budget algorithms, the budget is not given in advance;
rather, the total number of times an arm can be pulled is determined
by the function $\Nstar$
which depends on
the risk~\eqref{eqn:risk}, the cost $c$, and the size of the current surviving set $S$. 
Similarly,  unlike algorithms for fixed confidence
BAI~\citep{jun2016top,jamieson2014lil}, the confidence intervals are carefully
chosen based on problem parameters, and not via
a prespecified failure probability target $\delta$.
This design allows \algname{} to adapt to the problem difficulty with respect to the gaps
and cost,
while simultaneously ensuring control over the worst-case risk.

\parahead{Upper bound}
We show that the above algorithm, with carefully chosen parameters, matches
the lower bounds in almost all regimes.
Specifically, for $\riskmi$, our algorithm matches the lower bound up to polylog factors for all values of the complexity parameter $H$.
For $\risksr$, we similarly match the lower bound up to polylog factors when $H$ is not too large. 
However, when $H\to\infty\,$, our upper bound scales as $\bigO(\log(K)(K\sigma^2 c)^{1/3})$, while the lower bound is $\Omega(\Delta_2)$, leaving an additive gap.

Despite this discrepancy in the $\risktwo$ case, we make two important observations.
First, we show that our algorithm is \emph{minimax optimal}; that is, the worst-case risk over all problem instances matches the worst-case lower bound up to logarithmic factors.
Second, the lower bound in the large $H$ regime is tight and cannot be improved: a
naive guessing algorithm---one that selects an arm without pulling any---achieves the
lower bound on certain problem instances in this region. 
However, such a policy performs poorly when $H$ is small, underscoring the value of our adaptive strategy.

\emph{Proof ideas.}
Our use of an elimination-style procedure allows us to guarantee that
we never eliminate the optimal arm with high probability, and also identify precisely when
highly suboptimal arms are guaranteed to be eliminated. Then, by choosing
$\Nstar(|S|)\asymp\bigO((\abs{S}c)^{-1})$ for $\riskone$ and
$\Nstar(|S|)\asymp\bigO(\sigma^{\nicefrac{2}{3}}(\abs{S}c)^{-\nicefrac{2}{3}})$ for
$\risktwo$, we ensure that \algname{} can both match the worst-case behavior of the
lower bound and adapt to easier problem settings where there are relatively few good
candidate arms.

\parahead{Empirical evaluation}
We corroborate these theoretical findings in simulations and in a real-world experiment on a drug discovery dataset.
We compare  to
fixed budget and confidence algorithms to show the deficiencies of naive
adaptations of existing BAI paradigms on this problem.

\subsection{Related work}
\parahead{BAI}
The multi-armed bandit (MAB) problem, first introduced by
Thompson~\cite{thompson1933likelihood}, has become a foundational framework for studying
the exploration-exploitation trade-off in sequential decision-making under uncertainty.
Within this framework, Best Arm Identification (BAI) focuses on identifying the arm with
the highest expected
reward~\citep{bubeck2009pure,kalyanakrishnan2010efficient,gabillon2012best,Bubeck2013Multi,LogLog_Karnin2013,jamieson2014lil,russo2016simple}.

BAI has primarily been studied under two paradigms: the fixed-budget and fixed-performance settings. In the fixed-budget setting, the objective is to
minimize the probability of
misidentification~\citep{Audibert2010,Kaufmann2012,Kaufmann2016ComplexityBAI,Carpentier2016_FBLB,Barrier2023_FBNP}, or alternatively, to
minimize the simple regret~\citep{bubeck2009pure,Bubeck2011PureExploration,Zhao2023_SR}.
In the fixed performance setting, the majority of the literature has focused
on achieving a target probability of misidentification (a.k.a fixed confidence
BAI)~\citep{EvenDar2002,Mannor_Tsitsiklis2004,EvenDar2006,jamieson2014lil,Garivier2016,jun2016top,Jourdan2022_FC}.
To the best of our knowledge, there is no prior work on minimizing the number of pulls
subject to a performance goal on the simple regret.

Our work builds on the extensive literature in this area. In particular, our algorithm
draws inspiration from racing-style methods developed for fixed-confidence
BAI~\citep{Maron1997_Racing,jamieson2014lil,jun2016top}, while our lower bounds rely on technical lemmas
from~\citet{Kaufmann2016ComplexityBAI}. Nevertheless, the problem we study departs meaningfully
from existing formulations, requiring new conceptual insights and analytical tools.

\parahead{Cost of arm pulls in MAB}
Several works have explored sampling costs in BAI.
\citet{CostRatioBAI_Xia} and \citet{CostRatioBAI_Qin2020} study identifying the arm with highest
reward-to-cost ratio, assuming both reward and cost are observed per sample, both in
fixed-budget and fixed-confidence settings.
In contrast, in our setting, once a final arm is selected, only its
expected reward---not its sampling cost---remains relevant.
\citet{RegretBAI_Degenne2019} and \citet{Yang2024BestArm} consider minimizing the cumulative regret~\citep{Robbins1952aspectssequential} while performing BAI, but this approach is not applicable when sampling costs are exogenous to rewards, as we consider in our setting.
\citet{CostAwareBAI_Kanarios2024} study minimizing cumulative cost (instead of the number
of pulls) in a fixed confidence setting, when the learner observes a stochastic cost
on each arm pull in addition to the reward.
Recent work in 
multi-Fidelity BAI~\citep{poiani2022_MF,wang2023_MF,Poiani2024OptimalMultiFidelity} allows a learner to choose to
incur different costs for varying magnitudes of accuracy.
The last two problem settings are distinctly different from ours.
Finally, some works~\citep{Bananidiyuru2018_BwK,CostCumRegret_Sinha2021}
address costs in the cumulative
regret setting, which is also distinct from our focus on BAI.

\parahead{Bayesian sequential testing in classical statistics}
\citet{BayesSolutions_Arrow1949} and \citet{BayesSolutions_Wald1950} study Bayesian
formulations of sequential binary hypothesis testing problems (e.g., $H_1: \mu_1 - \mu_2
=\Delta$ vs.\ $H_2: \mu_1 - \mu_2 =-\Delta$), where the learner must balance the cost of
incorrect decisions against the cost of continued testing. They show that the
Bayes-optimal procedure for such problems is the sequential probability ratio test (SPRT)
of \citet{SPRT_Wald1945}, with optimal thresholds determined by solving complex implicit
equations that depend on the specific problem parameters. A number of works
\citep{sobel1953essentially,chernoff1965sequential,Bather_Walker_1962,SequentialReview_Lai}
have extended this study to the more general composite hypothesis testing framework ($H_1:
\mu_1-\mu_2>0$ vs.\ $H_2: \mu_1-\mu_2\leq0$).
While there are similarities to our proposed setting,
their analyses have been restricted to developing procedures that are only
asymptotically Bayes-optimal and only hold in the case of exponential families.

\parahead{Paper organization}
The remainder of this paper is organized as follows.
In~\S\ref{sec:2-arm}, we study the problem in the 2-arm setting.
This new formalism for BAI introduces novel intuitions which are best illustrated
in the two arm setting.
In~\S\ref{sec:karm}, we present our algorithm and main results in the
$K$-arm setting.
Finally, in~\S\ref{sec:experiments}, we evaluate our methods on simulations and show that it outperforms traditional BAI methods on this problem.

%% file: two_arm.tex
\section{Two-Arm Setting}\label{sec:2-arm}

To build intuition for this problem, we first study the $K=2$ setting.
Let $\cP(\R)$ denote all probability measures on $\R$, and let $G_\sigma = \big\{\lambda
\in \cP(\R) : \forall\, t > 0,\;$ $ \Prob_\lambda\left(X - \E_\lambda[X] > t\right) \leq
\exp\left(-\nicefrac{t^2}{2\sigma^2}\right) \big\}$ denote all
$\sigma$-sub-Gaussian probability distributions.
Let $\cM$, defined below in~\eqref{eq:2armclass}, denote the class of two-armed bandit
models with $\sigma$-sub-Gaussian rewards; recall that $\mu_i = \E_{\nui}[X]$. For a
given gap $\Delta \geq 0$, let $\cM_\Delta$, defined below, denote the subclass of models with
sub-optimality gap $\Delta$. We have:
\begin{align}\label{eq:2armclass}
\hspace{-0.06in}
    \cM \defeq \left\{\nu = (\nuone, \nutwo) : \nuone, \nutwo \in G_\sigma; \muone,\mutwo\in[0,B] \right\}, \quad
    \;
    \cM_\Delta \defeq \left\{ \nu \in \cM : \muone - \mutwo = \Delta \right\}.
\end{align}
In~\S\ref{sec:2armmi}, we begin by studying $\riskmi$ in~\eqref{eqn:risk}, which uses the
probability of misidentification as the performance criterion. In~\S\ref{sec:2armsr},
we then consider $\risksr$, which instead uses the simple regret.
Unless otherwise stated, 
all results in this section will be corollaries of more general results
in~\S\ref{sec:karm}.

\subsection{Probability of misidentification in the two-arm setting}\label{ss:2-arm-pm}
\label{sec:2armmi}

\parahead{Lower bound}
We begin with a gap-dependent lower bound applicable to any policy on this problem.

	\begin{corollary}<thm:k-arm-lb-pm>[Corollary of Theorem~\ref{thm:k-arm-lb-pm}, Lower bound on $\riskmi$]
%	\begin{corollary}[Corollary of Theorem~\ref{thm:k-arm-lb-pm}, Lower bound on $\riskmi$]
		\label{cor:2-arm-lb-pm}
		Fix a gap $\Delta>0$ and the cost $c$ per arm pull.
     Then, for any policy $\pi\,,$ we have
		\begin{align}\label{eq:2-arm-lb-pm}
			\sup_{\nu\in\cM_\Delta} \riskmi(\pi,\nu) \geq \hardmi(\Delta) 
    \defeq \begin{cases}
				\frac{\sigma^2c}{4\Delta^2}\log\left(\frac{e\Delta^2}{\sigma^2c}\right)\,,
& \text{if }\; \Delta\geq\sqrt{\sigma^2c}, \\
				\nicefrac{1}{4}\,, &\text{if }\; \Delta<\sqrt{\sigma^2c}.
			\end{cases}
		\end{align}
	\end{corollary}

It is instructive to compare the above result with lower bounds for fixed confidence BAI.
As in the fixed confidence setting~\citep{Kaufmann2016ComplexityBAI}, we observe that
for large $\Delta$, the lower bound exhibits a familiar dependence on
$\sigma^2\Delta^{-2}$, indicating that the problem becomes easier as the gap increases.
Our bound also depends on the cost $c$ and
includes a logarithmic term in $\Delta^2(\sigma^2 c)^{-1}$. Notably, when the gap is
small, our setting departs from fixed confidence behavior: the lower bound undergoes a
phase transition and saturates at a constant value of $\nicefrac{1}{4}$, rather than
continuing to increase with $\Delta^{-2}$.

\insertTwoArmBounds

\parahead{An oracular policy}
To build intuition towards designing a policy, it is worth considering the behavior of
an ``oracular'' policy which knows the gap $\Delta$ but does not know which of the
two arms is optimal.
Recall that it requires approximately $N(\Delta,\delta) \in
\bigO(\sigma^2\Delta^{-2}\log(\nicefrac{1}{\delta}))$
samples to separate
two sub-Gaussian distributions whose means are $\Delta$ apart~\citep{SPRT_Wald1945,Jennison1982_Elim,Kaufmann2016ComplexityBAI} with probability
at least $1-\delta$.
Hence, if we pull both arms $N(\Delta, \delta)$ times, we will incur a penalty of
$\delta + \bigO(\sigma^2c\Delta^{-2}\log(\nicefrac{1}{\delta}))$.
By optimally choosing $\delta \in \bigO(\sigma^2c\Delta^{-2})$, we find that we need to
pull each arm $\N(\Delta) \in \bigO(\sigma^2\Delta^{-2}\log(\Delta^2(\sigma^2c)^{-1}))$ times.
However, the above expression is non-negative only when $\Delta
\geq \bigOmega(\sqrt{\sigma^2c})$.
Intuitively, if $\Delta$ is very small, the policy will need to incur a large cost to
separate the two arms.
If the policy knows $\Delta$ it is better off randomly guessing an arm instead of
incurring this large cost.
This intuition leads to the following policy and theoretical result.
Its proof, which is straightforward, is given in Appendix~\ref{app:oracle-pf}.

	\begin{proposition}\label{prop:2-arm-oracle-pm}
		Let $\piDelta$ be the policy which pulls each arm
$\max\left\{0,\ceil{\frac{4\sigma^2}{\Delta^2}\log\left(\frac{\Delta^2}{8\sigma^2c}\right)}\right\}$
times.
If it pulls $0$ times, it will choose an arm uniformly at random, and otherwise,
it outputs the empirically largest
arm (breaking ties arbitrarily). Then, letting $\hardmi(\Delta)$ be as in
\eqref{eq:2-arm-lb-pm}, we have,
		\begin{align*}
			\sup_{\nu\in\cM_\Delta}\riskmi(\piDelta,\nu) \leq 32\hardmi(\Delta) + 2c 
    \in \bigO(\hardmi(\Delta))
		\end{align*}
	\end{proposition}

As we see, and illustrated in~Fig~\ref{fig:boundillus},
this bound matches the lower bound up to constant factors.%
\footnote{Proposition~\ref{prop:2-arm-oracle-pm} includes an additive penalty corresponding to
	the cost of two extra pulls, and a similar
	additive term appears in all upper bounds. 
	This is unavoidable in general, as even as 
	$\Delta \rightarrow \infty$, each arm must be pulled at least once to identify it. While
	this can be formally incorporated in the lower bound, we omit it for simplicity.}
To design a policy when $\Delta$ is unknown, we will leverage the above intuition.
We will also draw inspiration from prior work on
racing-style algorithms~\citep{Paulson_1964,Maron1997_Racing},
which have shown that sequentially pulling arms and eliminating them based on
confidence intervals can match oracular policies up to log factors in the fixed confidence
setting.

\parahead{A policy for $\riskmi$}
We will let $\delta$ be a confidence hyperparameter, aiming to output the
optimal arm with probability at least $1-\delta$.
However, to avoid over-pulling when the gap $\Delta$ is too small, we also incorporate a
hyperparameter
$\Nstar$, which is a limit on the total amount of times we are willing to pull
\emph{each} arm.
Intuitively, we know the cost grows linearly in the number of pulls, but the probability
of misidentification decays exponentially, so there is a point where the trade-off between
the cost of pulling and the increased precision these pulls provide no longer favors
continuing to pull.

Our approach proceeds in epochs of sampling both arms once and comparing the difference
between the empirical averages of the two arms against a Hoeffding confidence bound at the
end of each epoch to test for separation. If the observed difference on any epoch is
larger than the confidence bound, it will exit and recommend the larger arm. Otherwise, it
will continue to sample each arm until reaching the $\Nstar$-th and epoch, where it will
return the arm with the larger empirical average even though they have not statistically
separated.
In the case of the 2-arm probability of
misidentification setting, we use $\Nstar = (2ec)^{-1}$ and $\delta =
c(1+2c\Nstar)^{-1}$.
Here,
we set $\Nstar$ to be the maximum number of times the oracular policy would
ever pull each arm for any $\Delta$.
The confidence
parameter $\delta$ is used to control the penalty of the policy on the event that the
policy's confidence interval for the gap does not contain the true gap.
We have described this
algorithm formally in the $K$-arm setting in Algorithm~\ref{alg:our-alg}.

As the corollary below demonstrates,
by careful choice of $\Nstar$ and $\delta$, we show that we can match the lower bound
in Corollary~\ref{cor:2-arm-racing-pm} up to $\log(\nicefrac{1}{c})$ factors,
for all values of $\Delta$.
Based on the relationship between algorithmic performance and lower bounds in the BAI
literature, we conjecture that this logarithmic gap is largely unavoidable, and could at
best be reduced to a log-log factor~\citep{LogLog_Karnin2013,jamieson2014lil,Kaufmann2016ComplexityBAI}.

	\begin{corollary}<thm:k-arm-racing-pm>[Corollary of Theorem~\ref{thm:k-arm-racing-pm}, \algname{} under $\riskone$]
%	\begin{corollary}
		\label{cor:2-arm-racing-pm}
		Let $\pi$ be the policy described
above using $N^* = (2ec)^{-1}$ and $\delta=c(1+2cN^*)^{-1}\,$. Then, letting $\hardmi(\Delta)$ be as in \eqref{eq:2-arm-lb-pm},
		\begin{align*}
			\sup_{\nu\in\cM_\Delta}\riskone(\pi,\nu) \leq
128\log\left(\frac{e+1}{(ec)^2}\right)\hardmi(\Delta) + 3c \in
\bigO\left(\log\left(\frac{1}{c}\right)\hardmi(\Delta)\right).
		\end{align*}
	\end{corollary}

This bound and its comparison to the lower bound are illustrated in Fig~\ref{fig:boundillus}. As we can see in Fig~\ref{fig:boundillus}, by our choice of $\Nstar$, our policy actually performs within a constant factor of the lower bound for small $\Delta$, and the $\log(\nicefrac{1}{c})$ factor is incurred mostly in the ``moderate'' $\Delta$ regime. After the sharp transition at the midpoint of the plot in Fig~\ref{fig:boundillus}, representing the point at which our algorithm is guaranteed to output the optimal arm before reaching $\Nstar$ epochs with high probability, we can also see that the comparison to the lower bound quickly improves until we again reach a constant factor mismatch.

\subsection{Simple regret in the two-arm setting}\label{ss:2-arm-l2r}
\label{sec:2armsr}

\parahead{Lower bound} We again begin by presenting a lower bound on this problem.

	\begin{corollary}<thm:k-arm-lb-sr>[Corollary of Theorem~\ref{thm:k-arm-lb-sr}, Lower bound on $\risktwo$]
%	\begin{corollary}[Corollary of Theorem~\ref{thm:k-arm-lb-sr}, Lower bound on $\risktwo$]
		\label{cor:2-arm-lb-sr}
		Fix a gap $\Delta>0$ and the cost $c$ per arm pull. Then, for any policy $\pi\,,$
		\begin{align}\label{eq:2-arm-lb-sr}
			\sup_{\nu\in\cM_\Delta}\risktwo(\pi,\nu) \geq \hardsr(\Delta) = \begin{cases}
				\frac{\sigma^2c}{4\Delta^2}\log\left(\frac{e\Delta^3}{\sigma^2c}\right)\,, & \text{if }\; \Delta\geq(\sigma^2c)^{\nicefrac{1}{3}} \\
				\nicefrac{\Delta}{4}\,, & \text{if }\;\Delta < (\sigma^2c)^{\nicefrac{1}{3}}
			\end{cases}
		\end{align}
		Additionally, taking the worst-case over all $\Delta\,,$ we have, for any policy $\pi\,,$
		\begin{align}\label{eq:2-arm-lb-mm}
			\sup_{\nu\in\cM}\risktwo(\pi,\nu) \geq \hardsrstar =  \frac{3}{8}\left(\frac{\sigma^2c}{e}\right)^{\nicefrac{1}{3}}
		\end{align}
	\end{corollary}

As in Corollary~\ref{cor:2-arm-lb-pm}, we observe a phase transition in the lower bound:
it is $\nicefrac{\Delta}{4}$ when the gap is small, and scales as $\bigOmega(\Delta^{-2})$ when the gap is
large. For what follows, we also state the minimax (worst-case) value of this
lower bound as a function of $\Delta$.
As we see, this minimax lower bound decreases as the arm-pull cost $c$ decreases.
In contrast, for $\riskmi$, the minimax lower bound is $\nicefrac{1}{4}\,$, and even a naive policy
that guesses an arm without any pulls incurs a penalty of only $\nicefrac{1}{2}\,$. However, for $\risksr$, even achieving the minimax lower bound
requires a well-designed policy.

	\parahead{An oracular policy}
To design such a policy, let us again consider  the behavior of an oracular policy
which knows $\Delta$.
The motivation behind the chosen number of samples is the same as before, but when pulling
the arms $N(\Delta,\delta)$ times, we now incur a penalty of
$\delta\Delta+\bigO(\sigma^2c\Delta^{-2}\log(\nicefrac{1}{\delta}))\,.$ Because of this change, we now wish to use
$\delta\in\bigO(\sigma^2c\Delta^{-3})\,,$ leading to the following result, mirroring
that of Proposition~\ref{prop:2-arm-oracle-pm}. Its proof
which is straightforward, is given in
Appendix~\ref{app:oracle-pf}.
	\begin{proposition}
		\label{prop:2-arm-oracle-sr}
		Let $\pistar$ be the policy which pulls each arm $\max\left\{0,\ceil{\frac{4\sigma^2}{\Delta^2}\log\left(\frac{\Delta^3}{8\sigma^2c}\right)}\right\}$ times. If it pulls them 0 times, it will choose an arm uniformly at random, and otherwise, outputs the empirically largest arm (breaking ties arbitrarily). Then, letting $\hardsr(\Delta)$ be as in \eqref{eq:2-arm-lb-sr} and $\hardsrstar$ as in \eqref{eq:2-arm-lb-mm},
		\begin{align*}
			\sup_{\nu\in\cM_\Delta}\risktwo(\pistar,\nu) \leq 32\hardsr(\Delta) + 2c\in\bigO(\hardsr(\Delta))\,,\quad
			\sup_{\nu\in\cM}\hardsr(\pistar,\nu) \leq 8\hardsrstar + 2c
		\end{align*}
	\end{proposition}
	
	\parahead{A policy for $\risktwo$}
	Our policy will proceed exactly as before, performing rounds of equal sampling until
either we reach a prespecified number of epochs or we are able to identify the optimal arm
with high probability. Like the oracular policy, though, the change in risk requires
updating our hyperparameters $\Nstar$ and $\delta$ to ensure that our algorithm still
performs well in this setting. We again motivate our choice of $\Nstar$ via the behavior
of the oracular policy, choosing
$\Nstar=(\nicefrac{3}{2e})(\nicefrac{\sigma}{c})^{\nicefrac{2}{3}}\,$. We also still use
$\delta$ as a tool to control the worst-case penalty when our confidence interval does not
contain the true gap, and thus we set $\delta=c(B+2c\Nstar)^{-1}\,$.

	\begin{corollary}<thm:k-arm-racing-sr>[Corollary of Theorem~\ref{thm:k-arm-racing-sr}, \algname{} under $\risktwo$]
%	\begin{corollary}[Corollary of Theorem~\ref{thm:k-arm-racing-sr}, Algorithm performance on $\risktwo$]
		\label{cor:2-arm-racing-sr}
		Let $\pi$ be the policy described above using $N^* = (\nicefrac{3}{2e})(\nicefrac{\sigma}{c})^{\nicefrac{2}{3}}$ and $\delta = c(B+2c\Nstar)^{-1}\,$. Then, letting $\hardsr(\Delta)$ be as in \eqref{eq:2-arm-lb-sr}, when $\Delta\geq(\sigma^2c)^{\nicefrac{1}{3}}$, we have,
		\begin{align*}
			\sup_{\nu\in\cM_\Delta}\risktwo(\pi,\nu) \leq
				128\log\left(\frac{3B\sigma^{\nicefrac{4}{3}}}{c^{\nicefrac{5}{3}}}\right)\hardsr(\Delta) + 3c \in \bigO\left(\log\left(\frac{1}{c}\right)\hardsr(\Delta)\right)
		\end{align*}
		When $\Delta<(\sigma^2c)^{\nicefrac{1}{3}}$, we instead have,
		\begin{align*}
			\sup_{\nu\in\cM_\Delta}\risktwo(\pi,\nu) \leq 4\hardsr(\Delta) + 2(\sigma^2c)^{\nicefrac{1}{3}}+3c \in\bigO\left(\hardsr(\Delta)+\poly(c)\right)
		\end{align*}
		Finally, letting $\hardsrstar$ be as in~\eqref{eq:2-arm-lb-mm}, taking the worst case over all $\Delta$, we have,
		\begin{align*}
			\sup_{\nu\in\cM}\risktwo(\pi,\nu) \leq 9\hardsrstar + 3c \in \bigO(\hardsrstar)
		\end{align*}
	\end{corollary}
	Here we see, when $\Delta\geq(\sigma^2c)^{\nicefrac{1}{3}}$, these results closely
mirror that of Corollary~\ref{cor:2-arm-racing-pm}, though the log-factor now
additionally scales with $B\sigma^2$. As illustrated in Fig~\ref{fig:boundillus}, this
log-factor primarily plays a role in the moderate $\Delta$ regime like in the case of
$\riskone$. Our bound and Fig~\ref{fig:boundillus} also further highlight the inherent
difficulty of designing a simultaneously minimax- and instance-optimal policy for
$\risktwo$, as it is impossible to match the lower bound as $\Delta\to0$ without performing fewer pulls even as the problem becomes more difficult. Illustrating why the instance-based lower bound cannot be improved in this regime, however, is the
policy which guesses an arm without any pulls in purple in Fig~\ref{fig:boundillus}.

%% file: k_arm.tex
\section{K-arm Setting}\label{sec:karm}
We now generalize our results to the $K$-arm setting. We begin by adapting the notation formalities for $K$ arms. We now let $\cM$, defined in~\eqref{eq:class-k-arm}, denote the class of $K$-armed bandit models with $\sigma$-sub-Gaussian rewards. Further, for a bandit model $\nu\in\cM$, assuming WLOG that we have $\mu_1\geq\mu_2\geq\cdots\mu_K$, we define the complexity measure $\cH(\nu)\defeq\sum_{k=2}^K\Delta_k^{-2}$, where $\Delta_k=\mu_1-\mu_k$ is the $k$-th largest suboptimality gap. For a given complexity $H>0$, let $\cM_H$, defined below, denote the subclass of models having complexity at most $H$. Thus, we define:
\begin{align}
	\label{eq:class-k-arm}
	\cM = \left\{\nu = (\nu_a)_{a=1}^K : \nu_a\in G_\sigma\,,\mu_a\in[0,B]\;\forall\;a\in[K]\right\}\,, \quad \cM_H = \left\{\nu\in\cM : \cH(\nu) \leq H\right\}
\end{align}
% Following the structure of \S~\ref{sec:2-arm}, we study $\riskone$ in \S~\ref{ss:k-arm-mi}
% and $\risktwo$ in \S~\ref{ss:k-arm-sr}. Here we state our theorems as general results
% holding for all choices of $K$, and formally present our proposed algorithm, \algname.
% the $K$-arm setting provides additional challenges for
% designing and analyzing an algorithm, but we are still able to achieve near-optimal
% performance.
% KK: The first two sentences do not say much. The remainder of the para says there is a
% challenge but  doesn't say what it is. How about this? 

As we will see,
while our hardness results extend naturally from two to $K$ arms, 
extending the intuitions for the algorithm design requires a more careful design
of the budget parameter $\Nstar$.

\subsection{Probability of misidentification in the K-arm setting}
\label{ss:k-arm-mi}
\parahead{Lower bound}
We now present the general $K$-arm lower bound result for $\riskone$.
\begin{theorem}
	\label{thm:k-arm-lb-pm}
	Fix a complexity $H>0$ and a cost per arm pull $c>0\,.$ Then, for any policy $\pi$,
	\begin{align}\label{eq:k-arm-lb-pm}
		\sup_{\nu\in\cM_H}\riskone(\pi,\nu) \geq \hardmi(H) = \begin{cases}
			\frac{\sigma^2cH}{4}\log\left(\frac{e}{\sigma^2cH}\right)\,, & \text{if }\; H\leq (\sigma^2c)^{-1} \\
			\nicefrac{1}{4}\,, & \text{if }\; H>(\sigma^2c)^{-1}
		\end{cases}
	\end{align}
\end{theorem}
Comparing this result to its Corollary~\ref{cor:2-arm-lb-pm} in the 2-arm setting, we observe the same phase transition, now in terms of the complexity, $H$. Using the definition of $H$, we note that it still occurs when $\Delta_k\asymp\bigO((\sigma^2c)^{-1})$, and it provides the same intuition: when at least some of the gaps are sufficiently close to zero (or if there are very many arms), the cost of separating them outweighs the decrease in the probability of misidentification, and it becomes optimal to guess the best arm without pulling.

\parahead{A policy for $\riskone$}
We present our proposed algorithm, \algname{}, in its full $K$-arm generality in
Algorithm~\ref{alg:our-alg}. To account for there now being $K$ arms, \algname{} maintains
a ``surviving set'' $S$ of arms that have not yet been determined to be sub-optimal, and
performs rounds of equal sampling of all arms in $S$. At the end of each round, it
compares the difference between the current largest empirical average in $S$
and each other arm in $S$, and eliminates them based on Hoeffding confidence intervals.
% removes those that are more than a Hoeffding confidence
% bound less than the maximum.
This continues until either there is only one arm remaining,
or the remaining arms have reached their maximum per-arm budget, at which point the arm
with the largest empirical average is returned.

In moving from the two arm to $K$-arm regimes, we once again encounter the issue of balancing performance and costs when selecting our per-arm budget. On
one hand, if we naively replace the division by 2 in $\Nstar$ in Corollary~\ref{cor:2-arm-racing-pm} with a division by $K$, then
we will fall short on performance when
there are many highly sub-optimal arms. However, if we keep the same budget for each arm
from the 2-arm setting, we will perform too many total pulls when there are many
near-optimal arms. 

To this end, we allow the per-arm budgets to \textit{adapt} to the problem complexity by letting $\Nstar$ increase as $\abs{S}$ decreases. This allows \algname{} to dedicate additional resources to separating the remaining arms as some are determined to be sub-optimal, but prevents the total possible number of pulls from scaling too quickly in $K$.  Inspired by the 2-arm setting, we let $\Nstar(k)=(kec)^{-1}$. Further, we still use the confidence $\delta$ to control the worst-case penalty when the confidence intervals do not contain the true gap, so we set $\delta=c(1+2c\log(K)\Nstar(2))^{-1}$.
The following theorem summarizes the key properties of \algname{} when applied to
$\riskone$.

\begin{theorem}
	\label{thm:k-arm-racing-pm}
	Let $\pi_{\rm \algname}$ be the policy defined in Algorithm \ref{alg:our-alg} using $\Nstar(k)=(kec)^{-1}$ and $\delta=c(1+2c\log(K)\Nstar(2))^{-1}\,.$ Then, letting $\hardmi(H)$ be as in \eqref{eq:k-arm-lb-pm}, we have,
	\begin{align*}
		\sup_{\nu\in\cM_H}\riskone(\pi_{\rm \algname},\nu) \leq 760\log(K)\log\left(\frac{K\log(K)}{ec^2}\right)\hardmi(H)+(K+1)c\,,
	\end{align*}
	which is $\in\bigO(\polylog(K,c^{-1})\hardmi(H))$.
\end{theorem}
As in the 2-arm case, we see in Theorem~\ref{thm:k-arm-racing-pm} that our policy is still
able to achieve performance within a polylogarithmic factor of the lower bound, with the
addition of the $\log(K)$ factor being due to the worst-case impact of our adaptive budget
updating.

\insertAlgoMain

\subsection{Simple regret in the K-arm setting}
\label{ss:k-arm-sr}
\parahead{Lower bound}
We now present our second lower bound, for $\risktwo$ in the general $K$-arm setting.
\begin{theorem}
	\label{thm:k-arm-lb-sr}
	Fix a complexity $H>0\,,$ a smallest gap $\Delta_2\geq0\,,$ and a cost per arm pull $c>0\,$. Then, for any policy $\pi\,,$ we have,
	\begin{align}
		\label{eq:k-arm-lb-sr}
		\sup_{\nu\in\cM_H}\risktwo(\pi,\nu) \geq \hardsr(H) = \begin{cases}
			\frac{c\sigma^2H}{4}\log\left(\frac{e\Deltatwo}{\sigma^2cH}\right)\,, & \text{if }\; H\Delta_2^{-1}\leq(\sigma^2c)^{-1} \\
			\nicefrac{\Deltatwo}{4}\,, & \text{if }\;H\Delta_2^{-1}>(\sigma^2c)^{-1}
		\end{cases}
	\end{align}
	Additionally, taking the worst case over all problem instances, we have, for any policy $\pi\,$,
	\begin{align}
		\label{eq:k-arm-lb-mm}
		\sup_{\nu\in\cM}\risktwo(\pi,\nu)\geq\hardsrstar = \frac{3}{8}\left(\frac{(K-1)\sigma^2c}{e}\right)^{\nicefrac{1}{3}}
	\end{align}
\end{theorem}
Looking at the bound presented in~\eqref{eq:k-arm-lb-sr}, we see that the phase transition
in this lower bound now jointly involves the total problem complexity and the smallest
gap.

\parahead{A policy for $\risktwo$}
Following the same intuition as in the probability of misidentification case, we again wish to allow $\Nstar$ to adapt to the problem complexity and increase as the surviving set of arms shrinks. Observing the minimax lower bound presented in~\ref{eq:k-arm-lb-mm}, though, we see that the maximum problem difficulty scales with $K^{\nicefrac{1}{3}}$, unlike the constant scaling in the case of $\riskone$. Thus, we wish for $\Nstar$ to scale with $K^{-\nicefrac{2}{3}}$ instead of $K^{-1}$, and so we choose $\Nstar(k)=(\nicefrac{3}{2e})\sigma^{\nicefrac{2}{3}}((k-1)c)^{-\nicefrac{2}{3}}\,$. Then, controlling for the worst-case performance again, we choose $\delta=c(B+eK^{\nicefrac{1}{3}}\log(K)\Nstar(2))^{-1}$.
\begin{theorem}
	\label{thm:k-arm-racing-sr}
	Let $\pi_{\rm \algname}$ be the policy defined in Algorithm~\ref{alg:our-alg} using $\Nstar(k)=(\nicefrac{3}{2e})\sigma^{\nicefrac{2}{3}}((k-1)c)^{-\nicefrac{2}{3}}$ and $\delta=c(B+eK^{\nicefrac{1}{3}}\log(K)\Nstar(2))^{-1}$. Then, letting $\hardsr$ be as in~\eqref{eq:k-arm-lb-sr}, when $H\Delta_2^{-1}\leq(\sigma^2c)^{-1}$,
	\begin{align*}
		\sup_{\nu\in\cM_H}\risktwo(\pi_{\rm \algname},\nu) \leq
		550\log(K)\log\left(\frac{K\log(K)B\sigma^{\nicefrac{4}{3}}}{c^{\nicefrac{5}{3}}}\right)\hardsr(H)+(K+1)c\,,
	\end{align*}
	which is $\in\bigO(\polylog(B,K,c^{-1})\hardsr(H))$.
	When $H\Delta_2^{-1}>(\sigma^2c)^{-1}$, we instead have,
	\begin{align*}
		\hardsr(H) + 4\log(K)(K\sigma^2c)^{\nicefrac{1}{3}} + (K+1)c \in \bigO(\hardsr(H)+\log(K)\poly(K,c))
	\end{align*}
	Finally, letting $\hardsrstar$ be as in~\eqref{eq:k-arm-lb-mm}, we have,
	\begin{align*}
		\sup_{\nu\in\cM}\risktwo(\pi_{\rm \algname},\nu) \leq 20\log(K)\hardsrstar + (K+1)c \in \bigO(\log(K)\hardsrstar)
	\end{align*}
\end{theorem}
In Theorem~\ref{thm:k-arm-racing-sr}, we see similar performance of \algname{} compared to the lower bound as in the two-arm case: we are able to achieve performance within polylogarithmic factors when the complexity is relatively low, and we incur an additive logarithmic and polynomial factor in $K$ and $c$ when the complexity is prohibitively high. Observing our comparison to the minimax bound, we see that our policy is still minimax-optimal, being only a logarithmic factor in $K$ beyond the lower bound.

%% file: experiments.tex
\section{Numerical Experiments}
\label{sec:experiments}
\parahead{Simulation Studies}
We now empirically compare our method against traditional fixed budget and fixed
confidence methods to demonstrate the ability of \algname{} to perform well across all
problem instances. We study the performance across a range of suboptimality gaps $\Delta$
for Gaussian and Bernoulli rewards in the two-arm setting using the cost $c=10^{-4}$.
In the Gaussian
setting, the arms have variance $\sigma^2=1$ with means $\pm\nicefrac{\Delta}{2}$, for
$\Delta\in[0.05,2]$; for Bernoulli arms, the means are $0.5\pm\nicefrac{\Delta}{2}$, for
$\Delta\in[0.01,0.95]$. Results are averaged across $10^5$ runs each with different random
seeds. We compare to Sequential Halving for fixed budget and elimination procedures using the optimized stopping rules of~\citep{Kaufmann2016ComplexityBAI} for fixed confidence.
We use budgets $T=10$ and $T=500$ and confidences of
$\delta=0.1$ and $\delta=0.01$ for comparison against relatively low and high confidence/budget choices.
We also include the oracular policies of \S~\ref{sec:2-arm} to provide a baseline of good performance.
As we can see in Fig~\ref{fig:exp-two-arm}, the fixed budget and confidence algorithms necessarily have some region of gaps where they perform sub-optimally: for the small budget, it is moderate $\Delta$ values, for the large budget, it is large $\Delta$ values, and for both confidences, it is small $\Delta$ values. On the contrary, our proposed algorithm exhibits uniformly good performance across all $\Delta$ values, which is preferable when $\Delta$ is unknown. In Appendix~\ref{app:experiments}, we provide further simulations in the $K$-arm setting.
\insertMainSimulation

\parahead{Drug Discovery Experiment}
To demonstrate the efficacy of our approach on a problem in practice, we present the results of a real data experiment on a drug discovery dataset. For this experiment, we take the results from Table 2 of~\citet{GENOVESE2013} on the efficacy of the drug secukinumab in patients with rheumatoid arthritis. They report outcomes for 237 patients assigned to one of 5 treatment groups (arms) and report the drug efficacy according to the American College of Rheumatology criteria ACR20, ACR50, and ACR70. We consider this data under 2 settings: 1.) a binary efficacy outcome, being whether a patient achieved at least ACR20 (1) or not (0), as this was the primary goal of~\citep{GENOVESE2013}; and 2.) a ``leveled'' efficacy outcome, where no improvement results in an outcome of 0, and ACR20, ACR50, and ACR70 are outcomes of 0.2, 0.5, and 0.7, respectively, approximating a continuous efficacy metric. We treat the proportions of patients reported in each category in Table 2 of~\citep{GENOVESE2013} as population proportions, and evaluate \algname, Sequential Halving, and an elimination procedure with confidence bounds $\sqrt{4\sigma^2n^{-1}\log(Kn\delta^{-1})}\,$ on $10^4$ runs in each setting, each with different random seeds. For the binary outcome setting, the means were $\mu=(0.537,0.469,0.465,0.360,0.340)$, and for the leveled outcome setting, the means were $\mu=(0.230,0.227,0.200,0.196,0.102)$, each presented in decreasing order (order was randomized during data generation). Because $K$ and the means are fixed in this setting, we choose to evaluate our performance across a range of values for $c\in[10^{-3},10^{-5}]$. In Fig~\ref{fig:exp-drug}, we can see that no other method uniformly outperforms \algname{} across all choices of $c$, again highlighting the ability of our method to adapt to the problem setting at hand.
\insertDrugDiscovery

%% file: conclusion.tex
\section{Conclusion}
\label{sec:conclusion}
We propose a novel framework for studying best arm identification. In many practical settings, the traditional fixed budget and confidence regimes do not nicely align with the objectives of practitioners. To fill this gap, our setting explicitly balances sampling costs and performance on the fly by minimizing a risk functional. We prove hardness results for this problem and provide an algorithm, \algname, which achieves near-optimal performance on nearly all problem instances.

\parahead{Future directions}
We believe our lower bound analysis for simple regret in the $K$-arm setting can be improved.
 Though our bounds are tight when suboptimality gaps are similar,
we believe the bounds can be tighter when they are different.
We also conjecture that the additive gap we observe in the simple regret setting is
unavoidable for algorithms which achieve the minimax risk.

%% file: ack.tex
This work was supported in part by NSF Awards IIS-2441796 and DMS-2023239.

%% file: app_prior_results.tex
\section{Results from Prior Works}
\label{app:prior-works}
\begin{lemma}[Lemma 18 of \cite{Kaufmann2016ComplexityBAI}]
	\label{lem:kauffman-com}
	Let $\nu$ and $\nu^\prime$ be two bandit models, and let $\tau$ be any stopping time with respect to $\cF_t\,,$ where $\cF_t=\sigma(A_1,X_1,\ldots,A_t,X_t)$ is the sigma-algebra generated by all bandit interactions. For every event $\cE\in\cF_\tau$ (i.e., $\cE$ such that $\cE\cap\{\tau=t\}\in\cF_t$),
	\begin{align*}
		\Prob_{\nu^\prime}(\cE)=\E_\nu[\1_{\cE}\exp(-L_\tau)]\,,
	\end{align*}
	where,
	\begin{align*}
		L_\tau = \sum_{a=1}^K\sum_{s=1}^{N_a(\tau)}\log\left(\frac{f_a(Y_{a,s})}{f_a^\prime(Y_{a,s})}\right)\,,
	\end{align*}
	where $Y_{a,s}$ is the $s$-th i.i.d.\ observation of the $a$-th arm and $f_a$ and $f_a^\prime$ are the distribution functions of the $a$-th arm under $\nu$ and $\nu^\prime\,,$ respectively.
\end{lemma}
\begin{lemma}[Lemma 4 of \cite{Bubeck2013BoundedRegret}]
	\label{lem:testing-lb}
	Let $\rho_0,\rho_1$ be two probability distributions supported on some set $\cX\,,$ with $\rho_1\ll\rho_0\,.$ Then, for any measurable function $\phi:\cX\to\{0,1\}\,,$ one has
	\begin{align*}
		\Prob_{X\sim\rho_0}(\phi(X)=1) + \Prob_{X\sim\rho_1}(\phi(X)=0)\geq\frac{1}{2}\exp(-\KL(\rho_0\miid\rho_1))\,.
	\end{align*}
\end{lemma}

%% file: app_lb_proof.tex
\section{Proof of Lower Bounds}
\label{app:low-bounds}
We begin with a Lemma that will be central to all of our lower bounds, which builds upon the work of~\cite{Kaufmann2016ComplexityBAI} to achieve a bound with the form of their Lemma 15 which admits a random stopping time.
\begin{lemma}
	\label{lem:our-com}
	Let $\nu$ and $\nu^\prime$ be two $K$-arm bandit models. Let $\pi$ be any policy with associated stopping time $\tau$ such that $\Prob(\tau<\infty)=1$ which outputs an arm $\Ihat\in[K]$ at time $\tau\,.$ Then for any $a\in[K]\,,$
	\begin{align*}
		\Prob_{\nu,\pi}(\Ihat\neq a)+\Prob_{\nu^\prime,\pi}(\Ihat=a) \geq
		\frac{1}{2}\exp\left(-\sum_{a=1}^K\E_{\nu,\pi}[N_a(\tau)]\KL(\nu_a\miid\nu_a^\prime)\right)\,,
	\end{align*}
	where $\Pnupi$ is the probability with respect to all randomness incurred by $\pi$ interacting with the bandit model $\nu$ and $N_a(t)=\sum_{s=1}^t\1_{\{a\}}(A_s)$ is the number of times arm $a$ has been pulled up to round $t\,.$
\end{lemma}
\begin{proof}
	Fix a policy $\pi\,.$ For ease of notation, we suppress the $\pi$ in the probabilities and expectations throughout the proof. Now, we begin by using Lemma~\ref{lem:kauffman-com} to prove that the distributions of $\Ihat$ under each bandit model are absolutely continuous with one another. Consider an event $\cE\in\cF_\tau$ such that $\Prob_{\nu^\prime}(\cE)=0\,.$ Then, because $e^{-x}>0$ for all $x\in\R\,,$ we must have $\Prob_\nu(\1_{\cE}\exp(-L_\tau)=0)=1$ by Lemma~\ref{lem:kauffman-com}. Further, by our supposition that $\Prob_\nu(\tau<\infty)=1\,,$ we must have $\Prob_\nu(\exp(-L_\tau)>0)=1\,,$ and so we have $\Prob_{\nu^\prime}(\cE)=0\implies\Prob_\nu(\cE)=0\,,$ and we achieve the reverse implication by symmetry. Now, consider the fact that we necessarily have $\{\Ihat=a\}\in\cF_\tau$ for all $a\in[K]$ by construction, and so if we denote by $\cL(\Ihat)$ and $\cL^\prime(\Ihat)$ the distributions of $\Ihat$ under $\nu$ and $\nu^\prime\,,$ respectively, then clearly $\cL^\prime(\Ihat)\ll\cL(\Ihat)\,.$ Thus, we can apply Lemma~\ref{lem:testing-lb} to show,
	\begin{align*}
		\Prob_\nu(\Ihat\neq a) + \Prob_{\nu^\prime}(\Ihat=a) \geq \frac{1}{2}\exp\left(-\KL(\cL(\Ihat)\miid\cL^\prime(\Ihat))\right)
	\end{align*}
	To conclude the proof, we need only upper bound $\KL(\cL(\Ihat)\miid\cL^\prime(\Ihat))$ by $\E_\nu[L_\tau]\,,$ which we know is equal to $\sum_{a=1}^K\E_\nu[N_a(\tau)]\KL(\nu_a\miid\nu^\prime_a)\,$ by an application of Wald's Lemma~\cite{WaldLemma}. By applying the conditional Jensen inequality to the statement of Lemma~\ref{lem:kauffman-com} and rearranging the terms, we know for any $\cE\in\cF_\tau\,,$ we have $\E_\nu[L_\tau\mid\cE]\geq\log\frac{\Prob_\nu(\cE)}{\Prob_{\nu^\prime}(\cE)}\,.$ Thus, letting $\cI=\{k\in[K]:\Prob_\nu(\Ihat=k)\neq0\}\,,$ we can write,
	\begin{align*}
		\E_\nu[L_\tau] 
			& = \sum_{k\in\cI}\E_\nu[L_\tau\mid\Ihat=k]\Prob_\nu(\Ihat=k) \\
			& \geq \sum_{k\in\cI}\log\left(\frac{\Prob_\nu(\Ihat=k)}{\Prob_{\nu^\prime=k}(\Ihat=k)}\right)\Prob_\nu(\Ihat=k) \\
			& = \KL(\cL(\Ihat)\miid\cL^\prime(\Ihat))\,,
	\end{align*}
	which concludes the proof.
\end{proof}
We now employ Lemma~\ref{lem:our-com} to prove Theorems~\ref{thm:k-arm-lb-pm} and~\ref{thm:k-arm-lb-sr} and their associated corollaries.
\begin{proof}[Proof of Theorem~\ref{thm:k-arm-lb-pm}]
	Fix $H>0\,,\sigma^2>0\,,c>0\,$, and a policy $\pi\,.$ Let $\nu$ be a Gaussian $K$-arm bandit model with means $\mu_1>\mu_2\geq\cdots\geq\mu_K$ and common variance $\sigma^2$ satisfying $\cH(\nu)=H\,.$ Then, it is easy to show by contradiction that there must exist some arm $a\in\{2,\ldots,K\}$ such that $\E_{\pi,\nu}[N_a(\tau)]\leq\frac{\E_{\pi,\nu}[\tau]}{\Delta_a^2H(\nu)}\,.$ Let $\nu^\prime$ be an alternative model with Gaussian arms having the same common variance $\sigma^2\,,$ where $\mu_k(\nu)=\mu_k(\nu^\prime)$ for all $k\neq a$ and $\mu_a(\nu^\prime)=\mu_a(\nu)+2\Delta_a$ so that arm $a$ is now the optimal arm. Clearly $\cH(\nu^\prime)\leq \cH(\nu)\,,$ and so we have $\nu,\nu^\prime\in\cM_H\,.$ Then, we can apply Lemma~\ref{lem:our-com} to show,
	\begin{align*}\label{eq:proof-mi-lb}
		\sup_{\nu\in\cM_H}\riskone(\pi,\nu)
		& \geq \max\left\{\riskone(\pi,\nu),\riskone(\pi,\nu^\prime)\right\} \\
		& \geq \frac{1}{2}\left(\riskone(\pi,\nu)+\riskone(\pi,\nu^\prime)\right) \\
		& = \frac{1}{2}\left(\Prob_{\nu,\pi}(\Ihat\neq1) + \Prob_{\nu^\prime,\pi}(\Ihat\neq a)\right) + 
		\frac{c}{2}\left(\E_{\nu,\pi}[\tau]+\E_{\nu^\prime,\pi}[\tau]\right) \\
		& \geq \frac{1}{2}\left(\Prob_{\nu,\pi}(\Ihat\neq1) + \Prob_{\nu^\prime,\pi}(\Ihat=1)\right) + 
		\frac{c}{2}\left(\E_{\nu,\pi}[\tau]+\E_{\nu^\prime,\pi}[\tau]\right) \\
		& \geq \frac{1}{4}\exp\left(-\sum_{k=1}^K\E_{\nu,\pi}[N_k(\tau)]\KL(\nu_k\miid\nu_k^\prime)\right) +
		\frac{c}{2}\left(\E_{\nu,\pi}[\tau]+\E_{\nu^\prime,\pi}[\tau]\right) \\
		\numberthis
		& \geq \frac{1}{4}\exp\left(-\frac{2\E_{\nu,\pi}[\tau]}{\sigma^2H}\right) +
		\frac{c}{2}\left(\E_{\nu,\pi}[\tau]+\E_{\nu^\prime,\pi}[\tau]\right)
	\end{align*}
	Here, we recognize the fact that \eqref{eq:proof-mi-lb} is convex in $\E_{\nu,\pi}[\tau]\,,$ and thus we provide a $\pi$-free lower bound by minimizing over $\Enupi[\tau],\E_{\nu^\prime,\pi}[\tau]\geq0\,,$ which is achieved by $\Enupi[\tau]=\max\{0,\frac{\sigma^2H}{2}\log(\nicefrac{1}{\sigma^2cH})\}$ and $\E_{\nu^\prime,\pi}[\tau]=0\,.$ Plugging in these values completes the proof.
\end{proof}

\begin{proof}[Proof of Theorem~\ref{thm:k-arm-lb-sr}]
	This proof proceeds nearly identically to the proof above. Again fix $H>0,\sigma^2>0,c>0$ and any policy $\pi\,$. Then, let $\nu$ and $\nu^\prime$ be the same Gaussian $K$-arm bandit models as in the previous proof, with optimal arms 1 and $a\in\{2,\ldots,K\}$, respectively. For notational clarity, let the suboptimality gaps $\Delta_2,\ldots,\Delta_K$ be with respect to $\nu$ and let $\Delta_1\equiv0$, so that $\mu_a(\nu^\prime)-\mu_k(\nu^\prime)=\Delta_a+\Delta_k$ for $k\neq a$. Then, again using Lemma 3, we have,
	\begin{align*}\label{eq:proof-sr-lb}
		\sup_{\nu\in\cM_H}\risktwo(\pi,\nu) 
			& \geq \max\left\{\risktwo(\pi,\nu),\risktwo(\pi,\nu^\prime)\right\} \\
			& \geq \frac{1}{2}\left(\risktwo(\pi,\nu),\risktwo(\pi,\nu^\prime)\right) \\
			& = \frac{1}{2}\left(\Enupi[\mu_1-\muIhat] + \E_{\nu^\prime,\pi}[\mu_a-\muIhat]\right) +
				\frac{c}{2}\left(\E_{\nu,\pi}[\tau]+\E_{\nu^\prime,\pi}[\tau]\right) \\
			& = \frac{1}{2}\left(\sum_{i=2}^K\Delta_i\Pnupi(\Ihat=i) + \sum_{j\neq a}(\Delta_a+\Delta_j)\Prob_{\nu^\prime,\pi}(\Ihat=j)\right)\\
				& \qquad +	\frac{c}{2}\left(\E_{\nu,\pi}[\tau]+\E_{\nu^\prime,\pi}[\tau]\right) \\
			& \geq \frac{\Delta_2}{2}\left(\Pnupi(\Ihat\neq1) + \Prob_{\nu^\prime,\pi}(\Ihat\neq a)\right) + 
				\frac{c}{2}\left(\E_{\nu,\pi}[\tau]+\E_{\nu^\prime,\pi}[\tau]\right) \\
			& \geq \frac{\Delta_2}{2}\left(\Pnupi(\Ihat\neq1) + \Prob_{\nu^\prime,\pi}(\Ihat=1)\right) +
				\frac{c}{2}\left(\E_{\nu,\pi}[\tau]+\E_{\nu^\prime,\pi}[\tau]\right) \\
			& \geq \frac{\Delta_2}{4}\exp\left(-\sum_{k=1}^K\E_{\nu,\pi}[N_k(\tau)]\KL(\nu_k\miid\nu_k^\prime)\right) +
				\frac{c}{2}\left(\E_{\nu,\pi}[\tau]+\E_{\nu^\prime,\pi}[\tau]\right) \\
			\numberthis
			& \geq \frac{\Delta_2}{4}\exp\left(-\frac{2\E_{\nu,\pi}[\tau]}{\sigma^2H}\right) +
				\frac{c}{2}\left(\E_{\nu,\pi}[\tau]+\E_{\nu^\prime,\pi}[\tau]\right)
	\end{align*}
	Once again,~\eqref{eq:proof-sr-lb} is convex in $\Enupi[\tau]$, so we can further lower bound~\eqref{eq:proof-sr-lb} by setting $\Enupi[\tau]=\max\{0,\frac{\sigma^2H}{2}\log(\Delta_2(\sigma^2cH)^{-1})\}$ and $\E_{\nu^\prime,\pi}[\tau]=0$, completing the proof of~\eqref{eq:k-arm-lb-sr}.
	
	To prove~\eqref{eq:k-arm-lb-mm}, we can consider a specific set of means satisfying this problem instance. Consider the instance where $\mu_1=\Delta$ and $\mu_2=\cdots=\mu_K=0$, so that $\Delta_2=\cdots=\Delta_K=\Delta$ for some $\Delta>0$ that we will specify later. With these means, we have $H=(K-1)\Delta^{-2}$. Then, using~\eqref{eq:k-arm-lb-sr}, we can write,
	\begin{align*}
		\sup_{\nu\in\cM_{(K-1)\Delta^{-2}}}\risktwo(\pi,\nu)
		& \geq\hardsr((K-1)\Delta^{-2}) \\
		& = \begin{cases}
			\frac{(K-1)\sigma^2c}{4\Delta^2}\log\left(\frac{e\Delta^3}{(K-1)\sigma^2c}\right)\,, 
				& \text{if } \Delta\geq((K-1)\sigma^2c)^{\nicefrac{1}{3}} \\
			\nicefrac{\Delta}{4}\,, & \text{if }\Delta<((K-1)\sigma^2c)^{\nicefrac{1}{3}}
		\end{cases}
	\end{align*}
	We can then find the $\Delta$ which maximizes this function, which occurs at $\Delta^\star=(\sqrt{e}(K-1)\sigma^2c)^{\nicefrac{1}{3}}$, which gives,
	\begin{align*}
		\sup_{\nu\in\cM}\risktwo(\pi,\nu)\geq\sup_{\nu\in\cM_{(K-1)(\Delta^\star)^{-2}}}\risktwo(\pi,\nu) \geq\hardsr((K-1)(\Delta^\star)^{-2}) = 
			\frac{3}{8}\left(\frac{(K-1)\sigma^2c}{e}\right)^{\nicefrac{1}{3}}
	\end{align*}
\end{proof}

\begin{proof}[Proof of Corollaries~\ref{cor:2-arm-lb-pm} and~\ref{cor:2-arm-lb-sr}]
	Recall that, when $K=2,\, H=\Delta^{-2}$ and $\Delta_2=\Delta\,$. The conclusions then follow directly from Theorems~\ref{thm:k-arm-lb-pm} and~\ref{thm:k-arm-lb-sr}.
\end{proof}

%% file: app_oracle_ub.tex
\section{Oracular Policy Proofs}\label{app:oracle-pf}
\begin{proof}[Proof of Proposition~\ref{prop:2-arm-oracle-pm}]
	Fix a gap $\Delta>0\,.$ Because samples from each arm are i.i.d.\ $\sigma$-sub-Gaussian, by equally sampling the arms, we have i.i.d. $\sqrt{2}\sigma$-sub-Gaussian observations of the gap $\Delta\,$. By a Hoeffding confidence bound, if $\pi_T$ pulls each arm a fixed number of times $\ceil{T}$ and outputs the empirically largest arm, then we have
	\begin{align*}
		\riskone(\pi_T,\nu) = \Prob(\hat{\Delta}_{\ceil{T}}<0) + 2c\ceil{T} \leq \exp\left(-\frac{T\Delta^2}{4\sigma^2}\right) + 2c(T+1)
	\end{align*}
	Plugging in the proposed number of pulls, we get $\riskone(\piDelta,\nu)\leq \frac{8\sigma^2c}{\Delta^2}\log\left(\frac{e\Delta^2}{8\sigma^2c}\right)+2c$ when $\Delta\geq\sqrt{8\sigma^2c}\,,$ and exactly $\riskone(\piDelta,\nu)=\nicefrac{1}{2}$ otherwise, as then the policy guesses the optimal arm uniformly at random. Multiplying \eqref{eq:2-arm-lb-pm} by 32 and adding $2c$ then clearly upper bounds this quantity.
\end{proof}
\begin{proof}[Proof of Proposition~\ref{prop:2-arm-oracle-sr}]
	This proof proceeds nearly identically to the previous. Again fix a gap $\Delta>0$, and consider that we can write,
	\begin{align*}
		\risktwo(\pi_T,\nu) = \Delta\Prob(\hat{\Delta}_{\ceil{T}}<0) + 2c\ceil{T} \leq \Delta\exp\left(-\frac{T\Delta^2}{4\sigma^2}\right) + 2c(T+1)
	\end{align*}
	Then, plugging in the proposed number of pulls, we get $\risktwo(\pistar,\nu)\leq\frac{8\sigma^2c}{\Delta^2}\log\left(\frac{e\Delta^3}{8\sigma^2c}\right)+2c$ when $\Delta\geq(8\sigma^2c)^{\nicefrac{1}{3}}$ and exactly $\risktwo(\pistar,\nu)=\nicefrac{\Delta}{2}$ otherwise, as then the policy guesses the optimal arm uniformly at random. Multiplying~\eqref{eq:2-arm-lb-sr} by 32 and adding $2c$ then clearly upper bounds this quantity. Further, maximizing this upper bound in terms of $\Delta$ (occurring at $\Delta=(8\sqrt{e}\sigma^2c)^{\nicefrac{1}{3}}$) yields,
	\begin{align*}
		\sup_{\nu\in\cM}\risktwo(\pistar,\nu) = \sup_{\Delta}\sup_{\nu\in\cM_\Delta}\risktwo(\pistar,\nu) \leq 3\left(\frac{\sigma^2c}{e}\right)^{\nicefrac{1}{3}} + 2c = 8\hardsrstar + 2c
	\end{align*}
\end{proof}

%% file: app_alg_ub.tex
\section{Upper Bounds for \algname}
\label{app:alg-ub}
We begin by presenting a number of technical lemmas allowing us to control the behavior of \algname{} and prove our desired upper bounds on its performance.
\begin{lemma}[Bound on total number of pulls]
	\label{lem:alg-pulls}
	For any bandit instance $\nu\,$, using $\Nstar(k) = (kec)^{-1}$ and  $\Nstar(k)=(\nicefrac{3}{2e})\sigma^{\nicefrac{2}{3}}((k-1)c)^{-\nicefrac{2}{3}}\,$, \algname{} satisfies,
	\begin{align*}
		\Enupi[\tau] \leq \frac{2\log(K)}{ec}\,, \qquad \Enupi[\tau] \leq \frac{3\log(K)(K\sigma^2)^{\nicefrac{1}{3}}}{2c^{\nicefrac{2}{3}}}\,,
	\end{align*}
	respectively.
\end{lemma}
\begin{proof}
	Let $\khat$ denote the index of the $k$-th arm eliminated by the algorithm. Then by construction, $\Enupi[N_{\khat}(\tau)] \leq \Nstar(K-k+1)\,$. Further, $\Enupi[N_{\Ihat}(\tau)] \leq \Nstar(2)\,$. Thus,
	\begin{align*}
		\Enupi[\tau] = \sum_{a=1}^K\Enupi[N_a(\tau)] \leq \Nstar(2)+\sum_{k=2}^K\Nstar(k)
	\end{align*}
	Then, apply the fact that $\nicefrac{1}{2}+\sum_{k=2}^Kk^{-1}\leq 2\log(K)$ and $1+\sum_{k=2}^K(k-1)^{-\nicefrac{2}{3}}\leq eK^{\nicefrac{1}{3}}\log(K)$ to prove the statements.
\end{proof}
\begin{lemma}[Elimination behavior]
	\label{lem:conf-elim}
	Consider a bandit instance $\nu$ satisfying, WLOG, $\muone\geq\mutwo\geq\cdots\geq\muK\,$. Let $n(t)$ be the epoch associated with time $t\,$. Define the good event,
	\begin{align*}
		G=\bigcap_{n(t)\leq n(\tau)}\bigcap_{k\in S\setminus\{1\}}\{\Delta_k\in(\muhat_1(n(t))-\muhat_k(n(t))-e_{n(t)},\muhat_1(n(t))-\muhat_k(n(t))+e_{n(t)})\}\,.
	\end{align*}
	Then,
	\begin{enumerate}
		\item $\Pnupi(G^c)\leq\delta$
		\item On $G\,$, $1\in S\;\forall\;n(t)\leq n(\tau)$ (\ie the optimal arm is never eliminated)
		\item On $G$, if $\Delta_k>\sqrt{\frac{16\sigma^2\log(\nicefrac{K\Nstar(k)}{\delta})}{\Nstar(k)}}$ for all $k\geq\ell\in\{2,\ldots,K\}$ and $\Nstar$ decreasing in $k\,$,
		\begin{align*}
			N_k(\tau) \leq \frac{16\sigma^2\log(\nicefrac{K\Nstar(k)}{\delta})}{\Delta_k^2} < \Nstar(k)\;\forall\;k\geq \ell
		\end{align*}
	\end{enumerate}
\end{lemma}
\begin{proof}
	\phantom{x} \\
	\parahead{Part 1} 
		Letting $Y_{a,s}$ denote the $s$-th i.i.d.\ observation from arm $a\,$, by assumption, $Y_{1,s}-Y_{k,s}$ are $\sqrt{2}\sigma$-sub-Gaussian random variables with mean $\Delta_k\,$. Thus, $\sqrt{\frac{4\sigma^2\log(\nicefrac{n}{\delta})}{n}}$ is a $\delta$-correct confidence interval width for $\Delta_k$ after $n$ observations using $\widehat{\Delta}_{k,n}=\muhat_1(n)-\muhat_k(n)$ as the point estimate~\citep{jamieson2014lil,Kaufmann2016ComplexityBAI}. Replacing $\delta$ by $\nicefrac{\delta}{K}$ and taking a union bound across all $k\in S\setminus\{1\}$ then proves 1.
	\parahead{Part 2}
		Consider that on $G\,$, $\muhat_k(n)-\muhat_1(n) \leq e_n-\Delta_k \leq e_n$ for all $k\neq 1\,$, which proves 2.
	\parahead{Part 3}
		We begin with arm $K$. By the supposition, on $G\,$, there exists $n<\Nstar(k)$ such that $\muhat_1(n)-\muhat_K(n)-e_n\geq\Delta_K-2e_n>0\,$, and thus $K\notin S$ for all $m>n\,$. Further, we can upper bound the $n$ at which this is true by $\frac{16\sigma^2\log(\nicefrac{K\Nstar(k)}{\delta})}{\Delta_k^2}$ by construction of $e_n\,$, and this quantity less than $\Nstar(K)$ by the supposition. Then, because $K\notin S$ at time $\Nstar(K)\,$, if $\Nstar$ is decreasing in $k\,$, the algorithm will not be forced to terminate at time $\Nstar(K)$ by number of epochs, only if all arms other than 1 have already been eliminated, under which the statement would hold anyway. We can then use the same construction for each $k=K-1,\ldots,\ell$, proving 3.
\end{proof}
\begin{lemma}[Bound on probability of misidentification on the good event]
	\label{lem:alg-prob}
	For any bandit instance $\nu$ satisfying $\muone\geq\mutwo\geq\cdots\geq\muK\,$, and $\Nstar$ decreasing in $k$, if $M\in\{2,\ldots,K\}$ is the smallest value such that for each $k=M+1,\ldots,K\,$, $\Delta_k>\sqrt{\frac{16\sigma^2\log(\nicefrac{K\Nstar(k)}{\delta})}{\Nstar(k)}}$ (if no $\Delta_k$ satisfy this, $M=K$), then $\Pnupi(\{\Ihat=j\}\cap G)=0$ if $j>M$ and $\Pnupi(\{\Ihat=j\}\cap G)\leq\exp(-\frac{\Nstar(M)\Delta_j^2}{4\sigma^2})$ otherwise.
\end{lemma}
\begin{proof}
	We begin with the case $j>M\,$. By Lemma~\ref{lem:conf-elim}, arm 1 is never eliminated by the algorithm and arms $j,j+1,\ldots,K$ are eliminated before round $\Nstar(j)$. Then, on $G$, \algname{} only terminates because either $S=\{1\}$ or $n(\tau)=\Nstar(\abs{S})\geq\Nstar(j)$, making $\Pnupi(\{\Ihat=j\}\cap G)=0\,$. Now consider $j\leq M$. Then,
	\begin{align*}
		\Pnupi(\{\Ihat=j\}\cap G)
			& = \Pnupi(\{\Ihat=j\}\cap G\cap\{j\in S\text{ at }\tau\}) + 
				\Pnupi(\{\Ihat=j\}\cap G\cap\{j\notin S\text{ at }\tau\}) \\
			& = \Pnupi(\{\Ihat=j\}\cap G\cap\{j\in S\text{ at }\tau\}) \\
			& = \Pnupi(\{\muhat_j(n(\tau))>\muhat_1(n(\tau))\}\cap G\cap\{j\in S\text{ at }\tau\}) \\
			& \leq \Pnupi(\{\muhat_j(\Nstar(M))>\muhat_1(\Nstar(M))\}\cap G\cap\{j\in S\text{ at }\tau\}) \\
			& = \Pnupi\left(\left\{\frac{1}{\Nstar(M)}\sum_{n=1}^{\Nstar(M)}(Y_{j,n}-Y_{1,n})>0\right\}
				\cap G\cap\{j\in S\text{ at }\tau\}\right) \\
			& \leq \Pnupi\left(\frac{1}{\Nstar(M)}\sum_{n=1}^{\Nstar(M)}(Y_{j,n}-Y_{1,n})>0\right) \\
			& \leq \exp\left(-\frac{\Nstar(M)\Delta_j^2}{4\sigma^2}\right)
	\end{align*}
\end{proof}
\begin{lemma}[Bound on simple regret on the good event]
	\label{lem:alg-sr}
	For any bandit instance $\nu$ satisfying $\muone\geq\mutwo\geq\cdots\geq\muK\,$, and $\Nstar$ decreasing in $k$, if $M\in\{2,\ldots,K\}$ is the smallest value such that for each $k=M+1,\ldots,K\,$, $\Delta_k>\sqrt{\frac{16\sigma^2\log(\nicefrac{K\Nstar(k)}{\delta})}{\Nstar(k)}}$ (if no $\Delta_k$ satisfy this, $M=K$), then,
	\begin{align*}
	\Enupi[(\muone-\muIhat)\1_G]\leq\sqrt{\frac{4\sigma^2}{\sqrt{e}\Nstar(M)}}
	\end{align*}
\end{lemma}
\begin{proof}
	We begin by relating the simple regret with the probability of misidentification by applying Lemma~\ref{lem:alg-prob},
	\begin{align*}
		\Enupi[(\muone-\muIhat)\1_G]
		& = \sum_{i=2}^K\Delta_k\Pnupi(\{\Ihat=k\}\cap G) \\
		& \leq \sum_{k=2}^M\Delta_k\Pnupi\left(\bigcap_{\ell=1}^{k-1}\{\muhat_k(n(\tau))>\muhat_\ell(n(\tau))\}\cap G\right)		
	\end{align*}
	Now, consider that for $k>\ell\geq2, \Pnupi(\{\muhat_k(n(\tau))>\muhat_\ell(n(\tau))\}\cap G)$ is maximized when $\mu_k=\mu_\ell$ and is equal to $\nicefrac{1}{2}$ when this is the case. Thus, again applying Lemma~\ref{lem:alg-prob}, we can write,
	\begin{align*}
		\Enupi[(\muone-\muIhat)\1_G] 
		& \leq \Delta_2 \sum_{k=2}^M\frac{\Pnupi(\{\muhat_2(n(\tau))>\muhat_1(n(\tau))\}\cap G)}{2^{k-1}} \\
		& \leq 2\Delta_2\Pnupi(\{\muhat_2(n(\tau))>\muhat_1(n(\tau))\}\cap G) \\
		& \leq 2\Delta_2\exp\left(-\frac{\Nstar(M)\Delta_2^2}{4\sigma^2}\right)
	\end{align*}
	Maximizing in terms of $\Delta_2$ then proves the statement.
\end{proof}
With this collection of technical lemmas providing control on the behavior of \algname, we are ready to prove Theorems~\ref{thm:k-arm-racing-pm} and~\ref{thm:k-arm-racing-sr}.
\begin{proof}[Proof of Theorem~\ref{thm:k-arm-racing-pm}]
	We break this proof into two cases. First, consider problems of complexity $H\leq(\sigma^2c)^{-1}$ with $\muone\geq\mutwo\geq\cdots\geq\muK$. Further, let $M\in\{1,\ldots,K\}$ be the smallest value such that for each $k=M+1,\ldots,K\,$, $\Delta_k>\sqrt{16ek\sigma^2c\log(\nicefrac{K\Nstar(k)}{\delta})}$ (if no $\Delta_k$ satisfy this, $M=K$). Then, by the definition of $H$, we can write
	\begin{align}
		\label{eq:pf-mi-lb}
		\hardmi(H) = \frac{\sigma^2cH}{4}\log\left(\frac{e}{\sigma^2cH}\right) \geq
		\frac{M-1}{64eM\log(\nicefrac{K\Nstar(M)}{\delta})} + \frac{\sigma^2c}{4}\sum_{k=M+1}^K\frac{1}{\Delta_k^2}
	\end{align}
	Now, if $M\geq2$, we apply Lemmas~\ref{lem:alg-pulls} and~\ref{lem:conf-elim} to show the following:
	\begin{align*}
		\label{eq:pf-mi-M-ub}
		\Pnupi(\{\Ihat\neq 1\}\cap G) + c\Enupi[\tau\1_G]
		& = \Pnupi(\{\Ihat\neq 1\}\cap G) + c\sum_{k=1}^K\Enupi[N_k(\tau)\1_G] \\
		\numberthis
		& \leq 1 + \frac{2\log(M)}{e} + 16\sigma^2c\sum_{k=M+1}^K\frac{\log(\nicefrac{K\Nstar(k)}{\delta})}{\Delta_k^2}
	\end{align*}
	If, in fact, $M=1\,$, then combining the results of Lemmas~\ref{lem:alg-pulls},~\ref{lem:conf-elim}, and~\ref{lem:alg-prob}, we can write
	\begin{align}
		\label{eq:pf-mi-M1-ub}
		\Pnupi(\{\Ihat\neq 1\}\cap G) + c\Enupi[\tau\1_G] \leq 16\sigma^2c\sum_{k=M+1}^K\frac{\log(\nicefrac{K\Nstar(k)}{\delta})}{\Delta_k^2}
	\end{align}
	Then, multiplying~\eqref{eq:pf-mi-lb} by $760\log(K)\log(\nicefrac{K\log(K)}{ec^2})$ and adding $Kc$ (to account for non-integer pulls) then upper bounds both~\eqref{eq:pf-mi-M-ub} and~\eqref{eq:pf-mi-M1-ub}. Now, consider the case where $H>(\sigma^2c)^{-1}$. Then $\hardmi(H)=\nicefrac{1}{4}$. Directly applying Lemma~\ref{lem:alg-pulls} gives us for all $H$,
	\begin{align*}
		\Pnupi(\{\Ihat\neq 1\}\cap G) + c\Enupi[\tau\1_G] \leq 1 + \frac{2\log(K)}{e}\leq 760\log(K)\log\left(\frac{K\log(K)}{ec^2}\right)\left(\frac{1}{4}\right)
	\end{align*}
	Finally, consider that by our choice of $\delta\,$, using Lemmas~\ref{lem:alg-pulls} and~\ref{lem:conf-elim}, regardless of the value of $H$, we have
	\begin{align*}
		\Pnupi(G^c)\left(\Pnupi(\Ihat\neq1\mid G^c)+c\Enupi[\tau\mid G^c]\right) \leq
		\delta\left(1+\frac{2\log(K)}{e}\right) \leq c
	\end{align*}
	This then proves the statement.
\end{proof}

\begin{proof}[Proof of Theorem~\ref{thm:k-arm-racing-sr}]
	This proof largely mirrors that of~\ref{thm:k-arm-racing-pm}. Again, first consider problems satisfying $H\Delta_2^{-1}\leq(\sigma^2c)^{-1}$ with $\muone\geq\mutwo\geq\cdots\geq\muK$, and let $M\in\{1,\ldots,K\}$ be the smallest value such that for each $k=M+1,\ldots,K\,$, $\Delta_k>\sqrt{(\nicefrac{32e}{3})((k-1)\sigma^2c)^{\nicefrac{2}{3}}\log(\nicefrac{K\Nstar(k)}{\delta})}$ (if no $\Delta_k$ satisfy this, $M=K$). Then,
	\begin{align}
		\label{eq:pf-sr-lb}
		\hardsr(H)\geq\frac{3(M-1)^{\nicefrac{1}{3}}(\sigma^2c)^{\nicefrac{1}{3}}}{128e\log(K\Nstar(M)\delta^{-1})} + \frac{\sigma^2c}{4}\sum_{k=M+1}^K\frac{1}{\Delta_k^2}
	\end{align}
	If $M\geq 2$, we apply Lemmas~\ref{lem:alg-pulls},~\ref{lem:conf-elim}, and~\ref{lem:alg-sr} to show
	\begin{align}
		\label{eq:pf-sr-M-ub}
		\begin{split}
		\Enupi[(\muone-\muIhat)\1_G] + c\Enupi[\tau\1_G] & \leq \sqrt{\frac{8\sqrt{e}}{3}}((M-1)\sigma^2c)^{\nicefrac{1}{3}}+\frac{3\log(M)}{2}(M\sigma^2c)^{\nicefrac{1}{3}} \\ & \qquad + 16\sigma^2c\sum_{k=2}^K\frac{\log(K\Nstar(k)\delta^{-1})}{\Delta_k^2}
		\end{split}
	\end{align}
	Additionally, if $M=1$, then,
	\begin{align}
		\label{eq:pf-sr-M1-ub}
		\Enupi[(\muone-\muIhat)\1_G] + c\Enupi[\tau\1_G] \leq 32\sigma^2c\sum_{k=2}^K\frac{\log(K\Nstar(k)\delta^{-1})}{\Delta_k^2}
	\end{align}
	Then, multiplying~\eqref{eq:pf-sr-lb} by $575\log(K)\log(K\log(K)B\sigma^{\nicefrac{5}{3}}c^{-\nicefrac{4}{3}})$ upper bounds both~\eqref{eq:pf-sr-M-ub} and~\eqref{eq:pf-sr-M1-ub}. Now, for the case where $H\Delta_2^{-1}>(\sigma^2c)^{-1}$ and for the worst-case comparison, we apply Lemmas~\ref{lem:alg-pulls} and~\ref{lem:alg-sr} to show for all $H$,
	\begin{align}
		\label{eq:pf-sr-max}
		\Enupi[(\muone-\muIhat)\1_G] + c\Enupi[\tau\1_G] \leq \sqrt{\frac{8\sqrt{e}}{3}}((K-1)\sigma^2c)^{\nicefrac{1}{3}}+\frac{3\log(K)}{2}(K\sigma^2c)^{\nicefrac{1}{3}}
	\end{align}
	We then have~\eqref{eq:pf-sr-max} upper bounded by $4\log(K)(K\sigma^2c)^{\nicefrac{1}{3}}$, and $\hardsr(H)\geq0$. We can also upper bound~\eqref{eq:pf-sr-max} by $20\log(K)\hardsrstar$. Finally, we never incur more than an additional $Kc$ risk due to integer pulls, and by choice of $\delta$,
	\begin{align*}
		\Pnupi(G^c)\left(\Enupi(\muone-\muIhat\mid G^c)+c\Enupi[\tau\mid G^c]\right) \leq
		\delta\left(B+\frac{3c\log(K)\sigma^{\nicefrac{2}{3}}}{ec^{\nicefrac{2}{3}}}\right) \leq c
	\end{align*}
	which proves all statements.
\end{proof}

Now, despite our 2-arm results being corollaries of their more general $K$-arm counterparts, we are able to provide tighter constants in Corollaries~\ref{cor:2-arm-racing-pm} and~\ref{cor:2-arm-racing-sr} by utilizing some more precise techniques that are not generally applicable in the $K$-arm case. For both cases, we apply Lemma~\ref{lem:conf-elim} in the 2-arm case to identify a $\Deltastar$ such that, for all $\Delta>\Deltastar$, the algorithm is guaranteed to identify the optimal arm before reaching $\Nstar$ samples per arm on the good event $G$. We then show that we simply need to find a multiplier which makes the lower bound larger than the upper bound at $\Deltastar$, and this multiplier will work for all other $\Delta$.
\begin{proof}[Proof of Corollary~\ref{cor:2-arm-racing-pm}]
	We begin by using Lemma~\ref{lem:conf-elim} to identify $\Deltastar=\sqrt{32e\sigma^2c\log\left(\frac{e+1}{(ec)^2}\right)}$, which, combined with Lemma~\ref{lem:alg-prob}, allows us to write,
	\begin{align}
		\label{eq:proof-2-arm-mi-ub}
		\sup_{\nu\in\cM_\Delta}\riskone(\pi,\nu) \leq \riskoneub(\Delta)\defeq\begin{cases}
			\exp\left(-\frac{\Delta^2}{8e\sigma^2c}\right) + \frac{1}{e} + 3c\,, & \text{if }\Delta\leq\Deltastar \\
			\frac{32\sigma^2c\log\left(\frac{e+1}{(ec)^2}\right)}{\Delta^2} + 3c\,, & \text{if }\Delta>\Deltastar
		\end{cases}
	\end{align}
	where the additive $3c$ term is to account for integer pulls for each of the 2 arms and an additional $c$ bound for the expected risk on $G^c$. Clearly, for any $a\geq128\log\left(\frac{e+1}{(ec)^2}\right)$,~\eqref{eq:proof-2-arm-mi-ub} is upper bounded by $a\hardmi(\Delta)+3c$ for all $\Delta>\Deltastar$. We then divide our analysis for the remaining $\Delta$ into two cases: when $\Delta\leq\sqrt{e\sigma^2c}$ and otherwise. First, when $\Delta\leq\sqrt{e\sigma^2c}$, 
	\begin{align*}
		\riskoneub(\Delta)\leq \frac{e+1}{e}+3c\,, \qquad \hardmi(\Delta)\geq \frac{1}{2e}\,,
	\end{align*}
	and so $\riskoneub(\Delta)\leq 8\hardmi(\Delta)$ for $\Delta\leq\sqrt{e\sigma^2c}$. Finally, we must consider $\Delta\in(\sqrt{e\sigma^2c},\Deltastar]$. We begin by comparing~\eqref{eq:proof-2-arm-mi-ub} and~\eqref{eq:2-arm-lb-pm} at $\Deltastar$, then we prove that this is sufficient. This gives us,
	\begin{align*}
		\riskoneub(\Deltastar) = \left(\frac{(ec)^2}{e+1}\right)^4 + \frac{1}{e}+3c\,, \qquad 
		\hardmi(\Deltastar) = \frac{\log\left(32e^2\log\left(\frac{e+1}{(ec)^2}\right)\right)}{128e\log\left(\frac{e+1}{(ec)^2}\right)}
	\end{align*}
	Supposing $c<\nicefrac{1}{4}$,\footnote{%
	Previously, we have not put any restriction on the value of $c$, but we have implicitly assumed $c\ll1$ by the construction of our problem setting. Consider that, under $\riskone$, if $c\geq\nicefrac{1}{4}$, one will perform uniformly best on all instances by simply guessing the optimal arm uniformly at random. We do not explicitly account for this behavior in our algorithm construction for simplicity, but it is unrealistic to let $c\geq\nicefrac{1}{4}$ in practical settings.}
	we can see that $\riskoneub(\Deltastar)\leq128\log\left(\frac{e+1}{(ec)^2}\right)\hardmi(\Deltastar)$. Finally, we conclude that this is sufficient to prove the statement by showing that $128\log\left(\frac{e+1}{(ec)^2}\right)\hardmi(\Delta)-\riskoneub(\Delta)$ is decreasing for $\Delta\in(\sqrt{e\sigma^2c},\Deltastar]$. We show this here:
	\begin{align*}
		\frac{\partial}{\partial\Delta} a\hardmi(\Delta)-\riskoneub(\Delta)
			& = -\frac{a\sigma^2c}{2\Delta^3}\log\left(\frac{\Delta^2}{\sigma^2c}\right) +
				\frac{\Delta}{4e\sigma^2c}\exp\left(-\frac{\Delta^3}{8e\sigma^2c}\right) \\
			& \leq -\frac{a\sigma^2c}{2\Delta^3} + \frac{\Delta}{4e\sigma^2c}\left(\frac{8e\sigma^2c}{\Delta^2}\right)^2 \\
			& = -\frac{a\sigma^2c}{2\Delta^3} + \frac{16e\sigma^2c}{\Delta^3}\,,
	\end{align*}
	which is $<0$ when $a>32e\,,$ which is true for $a=128\log\left(\frac{e+1}{(ec)^2}\right)$. Thus, we have proven $\forall\;\Delta$,
	\begin{align*}
		128\log\left(\frac{e+1}{(ec)^2}\right)\hardmi(\Delta) \geq \riskoneub(\Delta) \geq \sup_{\nu\in\cM_\Delta}\riskone(\pi,\nu)
	\end{align*}
\end{proof}
\begin{proof}[Proof of Corollary~\ref{cor:2-arm-racing-sr}]
	We follow the same general proof strategy as in the previous proof. We again apply Lemma~\ref{lem:conf-elim} to identify $\Deltastar = (\sigma^2c)^{\nicefrac{1}{3}}\sqrt{\nicefrac{(8e)}{3}\log(\nicefrac{2\Nstar}{\delta})}$ and combine it with Lemma~\ref{lem:alg-prob} to write,
	\begin{align}
		\label{eq:proof-2-arm-sr-ub}
		\sup_{\nu\in\cM_\Delta}\risktwo(\pi,\nu) \leq \risktwoub(\Delta)\defeq \begin{cases}
			\Delta\exp\left(-\frac{3\Delta^2}{8e(\sigma^2c)^{\nicefrac{2}{3}}}\right) + \frac{3}{e}(\sigma^2c)^{\nicefrac{1}{3}} + 3c\,, &
				\text{if }\Delta\leq\Deltastar \\
			\frac{32\sigma^2c\log(\nicefrac{2\Nstar}{\delta})}{\Delta^2} + 3c\,, & \text{if }\Delta>\Deltastar
		\end{cases}
	\end{align}
	First, when $\Delta<(\sigma^2c)^{\nicefrac{1}{3}}$, clearly $\risktwoub(\Delta)\leq 4\hardmi(\Delta)+2(\sigma^2c)^{\nicefrac{1}{3}}+3c$ by~\eqref{eq:proof-2-arm-sr-ub}. Then, noting that $\frac{3B\sigma^{\nicefrac{4}{3}}}{c^{\nicefrac{5}{3}}}\geq\frac{2\Nstar}{\delta}$, we can clearly see that $\risktwoub(\Delta)\leq128\log\left(\frac{3B\sigma^{\nicefrac{4}{3}}}{c^{\nicefrac{5}{3}}}\right)\hardsr(\Delta)+3c$ for $\Delta>\Deltastar$. To prove this same bound for $\Delta\in[(\sigma^2c)^{\nicefrac{1}{3}},\Deltastar]$, we follow the same technique as in the previous proof. First, when $\Delta\in[(\sigma^2c)^{\nicefrac{1}{3}},(\sqrt{e}\sigma^2c)^{\nicefrac{1}{3}}]$,
	\begin{align*}
		\risktwoub(\Delta)\leq (\sigma^2c)^{\nicefrac{1}{3}}\left[\exp\left(\frac{1}{6}-\frac{3}{8e^{\nicefrac{2}{3}}}\right) + \frac{3}{e}\right]+3c\,,
		\qquad \hardsr(\Delta) \geq \frac{(\sigma^2c)^{\nicefrac{1}{3}}}{4}\,,
	\end{align*}
	and thus $\risktwoub(\Delta)\leq 9\hardsr(\Delta)+3c$ for $\Delta\in[(\sigma^2c)^{\nicefrac{1}{3}},(\sqrt{e}\sigma^2c)^{\nicefrac{1}{3}}]$. To prove the bound for $\Delta\in((\sqrt{e}\sigma^2c)^{\nicefrac{1}{3}},\Deltastar]$, we again compare the two at $\Deltastar$ and then show that the difference between the functions is decreasing in this range of $\Delta$, and thus this is sufficient. At $\Deltastar$, we have,
	\begin{align*}
		\risktwoub(\Deltastar) & = \frac{(\sigma^2c)^{\nicefrac{1}{3}}\sqrt{\frac{32}{3e}\log(\nicefrac{2\Nstar}{\delta})}}{(\nicefrac{2\Nstar}{\delta})^4} + \frac{3}{e}(\sigma^2c)^{\nicefrac{1}{3}}+3c \leq \frac{5}{e}(\sigma^2c)^{\nicefrac{1}{3}} +3c \\
		\hardsr(\Deltastar) & = \frac{3(\sigma^2c)^{\nicefrac{1}{3}}}{128e\log(\nicefrac{2\Nstar}{\delta})}
			\log\left(\frac{32e^{\nicefrac{5}{2}}}{3^{\nicefrac{3}{2}}}\log^{\nicefrac{3}{2}}(\nicefrac{2\Nstar}{\delta})\right) \geq
			\frac{9(\sigma^2c)^{\nicefrac{1}{3}}}{128e\log(\nicefrac{2\Nstar}{\delta})}
	\end{align*}
	Thus, we have $\risktwoub(\Deltastar)\leq128\log\left(\frac{3B\sigma^{\nicefrac{4}{3}}}{c^{\nicefrac{5}{3}}}\right)\hardsr(\Deltastar)+3c$. We then conclude this portion of the proof by showing $128\log\left(\frac{3B\sigma^{\nicefrac{4}{3}}}{c^{\nicefrac{5}{3}}}\right)\hardsr(\Delta)-\risktwoub(\Delta)$ is decreasing for $\Delta\in((\sqrt{e}\sigma^2c)^{\nicefrac{1}{3}},\Deltastar]$. We show this here:
	\begin{align*}
		\frac{\partial}{\partial\Delta} a\hardsr(\Delta) - \risktwoub(\Delta) 
		= -\frac{a\sigma^2c}{4\Delta^3}\log\left(\frac{\Delta^6}{e(\sigma^2c)^2}\right) - \exp\left(-\frac{3\Delta^2}{8e(\sigma^2c)^{\nicefrac{2}{3}}}\right)\left(1-\frac{3\Delta^2}{4e(\sigma^2c)^{\nicefrac{2}{3}}}\right)
	\end{align*}
	This is $<0$ for any $a\geq0$ when $\Delta\in((\sqrt{e}\sigma^2c)^{\nicefrac{1}{3}},\sqrt{\nicefrac{4e}{3}}(\sigma^2c)^{\nicefrac{1}{3}})$. When $\Delta\in[\sqrt{\nicefrac{4e}{3}}(\sigma^2c)^{\nicefrac{1}{3}},\Deltastar]$,
	\begin{align*}
		\frac{\partial}{\partial\Delta} a\hardsr(\Delta) - \risktwoub(\Delta)
		& \leq -\frac{a\sigma^2c}{2\Delta^3} + \frac{3\Delta^2}{4e(\sigma^2c)^{\nicefrac{2}{3}}}\left(\frac{8e(\sigma^2c)^{\nicefrac{2}{3}}}{3\Delta^2}\right)^{\nicefrac{5}{2}} \\
		& = -\frac{a\sigma^2c}{2\Delta^3} + \frac{\sigma^2c}{4\Delta^3}\left(8^{\nicefrac{5}{2}}\left(\frac{e}{3}\right)^{\nicefrac{3}{2}}\right)\,,
	\end{align*}
	which is $<0$ for any $a>78$, and in particular, $a=128\log\left(\frac{3B\sigma^{\nicefrac{4}{3}}}{c^{\nicefrac{5}{3}}}\right)$. Now, all that is left to prove is the worst-case comparison with $\hardsrstar$. We can show this simply by considering that~\eqref{eq:proof-2-arm-sr-ub} is maximized at $\Delta=\sqrt{\nicefrac{4e}{3}}(\sigma^2c)^{\nicefrac{1}{3}}\,$, where it takes value $(\sqrt{\nicefrac{4}{3}}+\nicefrac{3}{e})(\sigma^2c)^{\nicefrac{1}{3}}+3c$, which is clearly upper bounded by $9\hardsrstar+3c$.
\end{proof}

%% file: app_experiments.tex
\section{Additional Experiments}
\label{app:experiments}
Here we provide additional $K$-arm experiments. All experiments were performed using a 3.7GHz AMD Ryzen 9 5900X 12-Core processor with 24 GB of memory. Total runtime across all experiments took approximately 7.5 hours, and safeguards were employed to prevent the fixed confidence algorithms from continuing to sample after already severely underperforming the other methods when the sub-optimality gaps were particularly small ($\nicefrac{10}{c}$ total samples allowed).

\parahead{K-arm simulations}
We now include a number of additional $K$-arm experiments to demonstrate that our algorithm continues to perform well compared to traditional fixed budget and confidence methods when we move beyond the simple 2-arm case. For all of our $K$-arm experiments, we choose to use Gaussian arms with $\sigma^2=1$ for simplicity. We begin with the ``1-sparse'' setting, where $\mu_1=\Delta$ and $\mu_k=0$ for all $k\neq1$, resulting in $H=(K-1)\Delta^{-2}$, for $\Delta\in[0.05,2]$ for the probability of misidentification performance penalty and $\Delta\in[0.05,3]$ for the simple regret performance penalty. We additionally vary $K$ among 8, 16, and 32. For these experiments, we average across $10^4$ runs each with different random seeds. As in \S~\ref{sec:experiments}, we compare to Sequential Halving~\citep{LogLog_Karnin2013} for fixed budget and we use an elimination, or ``racing,'' procedure for fixed confidence, with confidence bounds $\sqrt{4\sigma^2n^{-1}\log(Kn\delta^{-1})}\,$. To extend to the $K$-arm case, our ``low'' budget is now $5K$, and our ``high'' budget is $250K$, which align with our choices of 10 and 500 in the 2-arm case. We still use $\delta=0.1$ and $\delta=0.01$ for our confidences. As we can see in Fig~\ref{fig:exp-k-sparse}, in the 1-sparse setting, \algname{} still enjoys uniformly good performance across the full range of $\Delta$, while the fixed budget and confidence approaches have some region where they perform sub-optimally.
\insertKSparse

To explore the performance of \algname{} and fixed confidence and budget approaches across a variety of problem structures, we additionally considered the ``linear decay'' setting, where we set $\mu_1 = \Delta_2$ and $\mu_k = -\Delta_2(\frac{k-2}{K-2})$ for $k\neq1$ so that the suboptimality gaps linearly increase from $\Delta_2$ to $2\Delta_2$. This results in $H \approx 0.5K\Delta_2^{-2}$. We again let $\Delta_2\in[0.05,2]$ for $\riskone$ and $\Delta_2\in[0.05,3]$ for $\risktwo$, average across $10^4$ runs each with a different random seed, and vary $K$ among 8, 16, and 32. As we can see in Fig~\ref{fig:exp-k-linear}, this setting provides similar results to the 1-sparse and 2-arm settings, with \algname{} performing well across the range of $\Delta_2$ values, while the other methods generally perform sub-optimally for some $\Delta_2$ values. 
\insertKLinear

%% file: main_arxiv.bbl
\begin{thebibliography}{52}
\providecommand{\natexlab}[1]{#1}
\providecommand{\url}[1]{\texttt{#1}}
\expandafter\ifx\csname urlstyle\endcsname\relax
  \providecommand{\doi}[1]{doi: #1}\else
  \providecommand{\doi}{doi: \begingroup \urlstyle{rm}\Url}\fi

\bibitem[Arrow et~al.(1949)Arrow, Blackwell, and
  Girshick]{BayesSolutions_Arrow1949}
K.~J. Arrow, D.~Blackwell, and M.~A. Girshick.
\newblock Bayes and {{Minimax Solutions}} of {{Sequential Decision Problems}}.
\newblock \emph{Econometrica}, 17\penalty0 (3/4):\penalty0 213--244, 1949.

\bibitem[Audibert et~al.(2010)Audibert, Bubeck, and Munos]{Audibert2010}
J.-Y. Audibert, S.~Bubeck, and R.~Munos.
\newblock Best arm identification in multi-armed bandits.
\newblock In \emph{23rd Annual Conference on Learning Theory}, pages 41--53.
  OmniPress, 2010.

\bibitem[Badanidiyuru et~al.(2018)Badanidiyuru, Kleinberg, and
  Slivkins]{Bananidiyuru2018_BwK}
A.~Badanidiyuru, R.~Kleinberg, and A.~Slivkins.
\newblock Bandits with {{Knapsacks}}.
\newblock \emph{J. ACM}, 65\penalty0 (3), 2018.

\bibitem[Barrier et~al.(2023)Barrier, Garivier, and Stoltz]{Barrier2023_FBNP}
A.~Barrier, A.~Garivier, and G.~Stoltz.
\newblock On {{Best-Arm Identification}} with a {{Fixed Budget}} in
  {{Non-Parametric Multi-Armed Bandits}}.
\newblock In \emph{Proceedings of {{The}} 34th {{International Conference}} on
  {{Algorithmic Learning Theory}}}, pages 136--181. PMLR, 2023.

\bibitem[Bather and Walker(1962)]{Bather_Walker_1962}
J.~A. Bather and A.~M. Walker.
\newblock Bayes procedures for deciding the sign of a normal mean.
\newblock \emph{Mathematical Proceedings of the Cambridge Philosophical
  Society}, 58\penalty0 (4):\penalty0 599–620, 1962.

\bibitem[Bechhofer(1958)]{Bechhofer11958SequentialMultipleDecision}
R.~E. Bechhofer.
\newblock A {{Sequential Multiple-Decision Procedure}} for {{Selecting}} the
  {{Best One}} of {{Several Normal Populations}} with a {{Common Unknown
  Variance}}, and {{Its Use}} with {{Various Experimental Designs}}.
\newblock \emph{Biometrics}, 14\penalty0 (3):\penalty0 408--429, 1958.

\bibitem[Bubeck et~al.(2009)Bubeck, Munos, and Stoltz]{bubeck2009pure}
S.~Bubeck, R.~Munos, and G.~Stoltz.
\newblock Pure exploration in multi-armed bandits problems.
\newblock In \emph{International Conference on Algorithmic Learning Theory},
  pages 23--37. Springer, 2009.

\bibitem[Bubeck et~al.(2011)Bubeck, Munos, and
  Stoltz]{Bubeck2011PureExploration}
S.~Bubeck, R.~Munos, and G.~Stoltz.
\newblock Pure exploration in finitely-armed and continuous-armed bandits.
\newblock \emph{Theoretical Computer Science}, 412\penalty0 (19):\penalty0
  1832--1852, 2011.

\bibitem[Bubeck et~al.(2013{\natexlab{a}})Bubeck, Perchet, and
  Rigollet]{Bubeck2013BoundedRegret}
S.~Bubeck, V.~Perchet, and P.~Rigollet.
\newblock Bounded regret in stochastic multi-armed bandits.
\newblock In \emph{Proceedings of the 26th {{Annual Conference}} on {{Learning
  Theory}}}, pages 122--134. PMLR, 2013{\natexlab{a}}.

\bibitem[Bubeck et~al.(2013{\natexlab{b}})Bubeck, Wang, and
  Viswanathan]{Bubeck2013Multi}
S.~Bubeck, T.~Wang, and N.~Viswanathan.
\newblock Multiple {{Identifications}} in {{Multi-Armed Bandits}}.
\newblock In \emph{Proceedings of the 30th {{International Conference}} on
  {{Machine Learning}}}, pages 258--265. PMLR, 2013{\natexlab{b}}.

\bibitem[Carpentier and Locatelli(2016)]{Carpentier2016_FBLB}
A.~Carpentier and A.~Locatelli.
\newblock Tight ({{Lower}}) {{Bounds}} for the {{Fixed Budget Best Arm
  Identification Bandit Problem}}.
\newblock In \emph{29th Annual Conference on Learning Theory}, pages 590--604.
  PMLR, 2016.

\bibitem[Chernoff(1965)]{chernoff1965sequential}
H.~Chernoff.
\newblock Sequential tests for the mean of a normal distribution iv (discrete
  case).
\newblock \emph{The Annals of Mathematical Statistics}, 36\penalty0
  (1):\penalty0 55--68, 1965.

\bibitem[Degenne et~al.(2019)Degenne, Nedelec, Calauzenes, and
  Perchet]{RegretBAI_Degenne2019}
R.~Degenne, T.~Nedelec, C.~Calauzenes, and V.~Perchet.
\newblock Bridging the gap between regret minimization and best arm
  identification, with application to {{A}}/{{B}} tests.
\newblock In \emph{Proceedings of the 22nd {{International Conference}} on
  {{Artificial Intelligence}} and {{Statistics}}}, pages 1988--1996. PMLR,
  2019.

\bibitem[Even-Dar et~al.(2002)Even-Dar, Mannor, and Mansour]{EvenDar2002}
E.~Even-Dar, S.~Mannor, and Y.~Mansour.
\newblock {PAC} bounds for multi-armed bandit and markov decision processes.
\newblock In \emph{Computational Learning Theory}, pages 255--270. Springer
  Berlin Heidelberg, 2002.

\bibitem[{Even-Dar} et~al.(2006){Even-Dar}, Mannor, and Mansour]{EvenDar2006}
E.~{Even-Dar}, S.~Mannor, and Y.~Mansour.
\newblock Action {{Elimination}} and {{Stopping Conditions}} for the
  {{Multi-Armed Bandit}} and {{Reinforcement Learning Problems}}.
\newblock \emph{Journal of Machine Learning Research}, 7\penalty0
  (39):\penalty0 1079--1105, 2006.

\bibitem[Gabillon et~al.(2012)Gabillon, Ghavamzadeh, and
  Lazaric]{gabillon2012best}
V.~Gabillon, M.~Ghavamzadeh, and A.~Lazaric.
\newblock Best arm identification: A unified approach to fixed budget and fixed
  confidence.
\newblock In \emph{Advances in Neural Information Processing Systems},
  volume~25, pages 3212--3220. Curran Associates, Inc., 2012.

\bibitem[Garivier and Kaufmann(2016)]{Garivier2016}
A.~Garivier and E.~Kaufmann.
\newblock Optimal {{Best Arm Identification}} with {{Fixed Confidence}}.
\newblock In \emph{29th Annual Conference on Learning Theory}, pages 998--1027.
  PMLR, 2016.

\bibitem[Geng et~al.(2021)Geng, Lin, Nair, Hao, Xiang, and
  Fan]{geng2021comparison}
T.~Geng, X.~Lin, H.~S. Nair, J.~Hao, B.~Xiang, and S.~Fan.
\newblock Comparison lift: Bandit-based experimentation system for online
  advertising.
\newblock In \emph{Proceedings of the AAAI Conference on Artificial
  Intelligence}, volume~35, pages 15117--15126, 2021.

\bibitem[Genovese et~al.(2013)Genovese, Durez, Richards, Supronik, Dokoupilova,
  Mazurov, Aelion, Lee, Codding, Kellner, Ikawa, Hugot, and
  Mpofu]{GENOVESE2013}
M.~C. Genovese, P.~Durez, H.~B. Richards, J.~Supronik, E.~Dokoupilova,
  V.~Mazurov, J.~A. Aelion, S.-H. Lee, C.~E. Codding, H.~Kellner, T.~Ikawa,
  S.~Hugot, and S.~Mpofu.
\newblock Efficacy and safety of secukinumab in patients with rheumatoid
  arthritis: A phase {{II}}, dose-finding, double-blind, randomised, placebo
  controlled study.
\newblock \emph{Annals of the Rheumatic Diseases}, 72\penalty0 (6):\penalty0
  863--869, 2013.

\bibitem[Jamieson and Talwalkar(2016)]{Jamieson2016NonstochasticBest}
K.~Jamieson and A.~Talwalkar.
\newblock Non-stochastic {{Best Arm Identification}} and {{Hyperparameter
  Optimization}}.
\newblock In \emph{Proceedings of the 19th {{International Conference}} on
  {{Artificial Intelligence}} and {{Statistics}}}, pages 240--248, 2016.

\bibitem[Jamieson et~al.(2014)Jamieson, Malloy, Nowak, and
  Bubeck]{jamieson2014lil}
K.~Jamieson, M.~Malloy, R.~Nowak, and S.~Bubeck.
\newblock lil’ucb: An optimal exploration algorithm for multi-armed bandits.
\newblock In \emph{27th Annual Conference on Learning Theory}, volume~35, pages
  423--439. PMLR, PMLR, 2014.

\bibitem[Jennison et~al.(1982)Jennison, Johnstone, and
  Turnbull]{Jennison1982_Elim}
C.~Jennison, I.~M. Johnstone, and B.~W. Turnbull.
\newblock Asymptotically optimal procedures for sequential adaptive selection
  of the best of several normal means.
\newblock In \emph{Statistical Decision Theory and Related Topics III}, pages
  55--86. Academic Press, 1982.

\bibitem[Jourdan et~al.(2022)Jourdan, Degenne, Baudry, {de Heide}, and
  Kaufmann]{Jourdan2022_FC}
M.~Jourdan, R.~Degenne, D.~Baudry, R.~{de Heide}, and E.~Kaufmann.
\newblock Top {{Two Algorithms Revisited}}.
\newblock In \emph{Advances in Neural Information Processing Systems},
  volume~35, pages 26791--26803. Curran Associates, Inc., 2022.

\bibitem[Jun et~al.(2016)Jun, Jamieson, Nowak, and Zhu]{jun2016top}
K.-S. Jun, K.~Jamieson, R.~Nowak, and X.~Zhu.
\newblock Top arm identification in multi-armed bandits with batch arm pulls.
\newblock In \emph{Proceedings of the 19th International Conference on
  Artificial Intelligence and Statistics}, volume~51, pages 139--148. PMLR,
  PMLR, 2016.

\bibitem[Kalyanakrishnan and Stone(2010)]{kalyanakrishnan2010efficient}
S.~Kalyanakrishnan and P.~Stone.
\newblock Efficient selection of multiple bandit arms: Theory and practice.
\newblock In \emph{Proceedings of the 27th International Conference on Machine
  Learning}, pages 511--518. PMLR, 2010.

\bibitem[Kalyanakrishnan et~al.(2012)Kalyanakrishnan, Tewari, Auer, and
  Stone]{kalyanakrishnan2012LUCB}
S.~Kalyanakrishnan, A.~Tewari, P.~Auer, and P.~Stone.
\newblock Pac subset selection in stochastic multi-armed bandits.
\newblock In \emph{Proceedings of the 29th International Coference on Machine
  Learning}, volume~12, page 227–234. Omnipress, 2012.

\bibitem[Kanarios et~al.(2024)Kanarios, Zhang, and
  Ying]{CostAwareBAI_Kanarios2024}
K.~Kanarios, Q.~Zhang, and L.~Ying.
\newblock Cost {{Aware Best Arm Identification}}.
\newblock \emph{Reinforcement Learning Journal}, 4:\penalty0 1533--1545, 2024.

\bibitem[Karnin et~al.(2013)Karnin, Koren, and Somekh]{LogLog_Karnin2013}
Z.~Karnin, T.~Koren, and O.~Somekh.
\newblock Almost {{Optimal Exploration}} in {{Multi-Armed Bandits}}.
\newblock In \emph{Proceedings of the 30th {{International Conference}} on
  {{Machine Learning}}}, pages 1238--1246. PMLR, 2013.

\bibitem[Kaufmann et~al.(2012)Kaufmann, Korda, and Munos]{Kaufmann2012}
E.~Kaufmann, R.~Korda, and R.~Munos.
\newblock Thompson sampling: An asymptotically optimal finite-time analysis.
\newblock In \emph{Algorithmic Learning Theory (ALT)}, pages 199--213. Springer
  Berlin Heidelberg, 2012.

\bibitem[Kaufmann et~al.(2016)Kaufmann, Capp{\'e}, and
  Garivier]{Kaufmann2016ComplexityBAI}
E.~Kaufmann, O.~Capp{\'e}, and A.~Garivier.
\newblock On the {{Complexity}} of {{Best-Arm Identification}} in {{Multi-Armed
  Bandit Models}}.
\newblock \emph{Journal of Machine Learning Research}, 17\penalty0
  (1):\penalty0 1--42, 2016.

\bibitem[Lai(1997)]{SequentialReview_Lai}
T.~L. Lai.
\newblock On optimal stopping problems in sequential hypothesis testing.
\newblock \emph{Statistica Sinica}, 7\penalty0 (1):\penalty0 33–51, 1997.

\bibitem[Li et~al.(2018)Li, Jamieson, DeSalvo, Rostamizadeh, and
  Talwalkar]{li2018hyperband}
L.~Li, K.~Jamieson, G.~DeSalvo, A.~Rostamizadeh, and A.~Talwalkar.
\newblock Hyperband: A novel bandit-based approach to hyperparameter
  optimization.
\newblock \emph{Journal of Machine Learning Research}, 18\penalty0
  (185):\penalty0 1--52, 2018.

\bibitem[Mannor and Tsitsiklis(2004)]{Mannor_Tsitsiklis2004}
S.~Mannor and J.~N. Tsitsiklis.
\newblock The sample complexity of exploration in the multi-armed bandit
  problem.
\newblock \emph{Journal of Machine Learning Research}, 5\penalty0
  (Jun):\penalty0 623–648, 2004.

\bibitem[Maron and Moore(1997)]{Maron1997_Racing}
O.~Maron and A.~W. Moore.
\newblock The {{Racing Algorithm}}: {{Model Selection}} for {{Lazy Learners}}.
\newblock \emph{Artificial Intelligence Review}, 11\penalty0 (1):\penalty0
  193--225, 1997.

\bibitem[Misra et~al.(2021)Misra, Liaw, Dunlap, Bhardwaj, Kandasamy, Gonzalez,
  Stoica, and Tumanov]{misra2021rubberband}
U.~Misra, R.~Liaw, L.~Dunlap, R.~Bhardwaj, K.~Kandasamy, J.~E. Gonzalez,
  I.~Stoica, and A.~Tumanov.
\newblock Rubberband: cloud-based hyperparameter tuning.
\newblock In \emph{Proceedings of the Sixteenth European Conference on Computer
  Systems}, pages 327--342. Association for Computing Machinery, 2021.

\bibitem[Paulson(1964)]{Paulson_1964}
E.~Paulson.
\newblock A sequential procedure for selecting the population with the largest
  mean from $k$ normal populations.
\newblock \emph{The Annals of Mathematical Statistics}, 35\penalty0
  (1):\penalty0 174–180, Mar. 1964.

\bibitem[Poiani et~al.(2022)Poiani, Metelli, and Restelli]{poiani2022_MF}
R.~Poiani, A.~M. Metelli, and M.~Restelli.
\newblock Multi-{{Fidelity Best-Arm Identification}}.
\newblock In \emph{Advances in Neural Information Processing Systems},
  volume~35, pages 17857--17870. Curran Associates, Inc., 2022.

\bibitem[Poiani et~al.(2024)Poiani, Degenne, Kaufmann, Metelli, and
  Restelli]{Poiani2024OptimalMultiFidelity}
R.~Poiani, R.~Degenne, E.~Kaufmann, A.~M. Metelli, and M.~Restelli.
\newblock Optimal {{Multi-Fidelity Best-Arm Identification}}.
\newblock In \emph{Advances in Neural Information Processing Systems},
  volume~37, pages 121882--121927. Curran Associates, Inc., 2024.

\bibitem[Qin et~al.(2020)Qin, Gan, Liu, Wu, Jin, and Fu]{CostRatioBAI_Qin2020}
Z.~Qin, X.~Gan, J.~Liu, H.~Wu, H.~Jin, and L.~Fu.
\newblock Exploring {{Best Arm}} with {{Top Reward-Cost Ratio}} in {{Stochastic
  Bandits}}.
\newblock In \emph{{{IEEE INFOCOM}} 2020 - {{IEEE Conference}} on {{Computer
  Communications}}}, pages 159--168, 2020.

\bibitem[Robbins(1952)]{Robbins1952aspectssequential}
H.~Robbins.
\newblock Some aspects of the sequential design of experiments.
\newblock \emph{Bulletin of the American Mathematical Society}, 58\penalty0
  (5):\penalty0 527–535, 1952.

\bibitem[Russo(2016)]{russo2016simple}
D.~Russo.
\newblock Simple bayesian algorithms for best arm identification.
\newblock In \emph{29th Annual Conference on Learning Theory}, volume~49, pages
  1417--1418. PMLR, 2016.

\bibitem[Siegmund(1985)]{WaldLemma}
D.~Siegmund.
\newblock \emph{Sequential Analysis}.
\newblock Springer New York, 1985.

\bibitem[Sinha et~al.(2021)Sinha, Sankararaman, Kazerouni, and
  Avadhanula]{CostCumRegret_Sinha2021}
D.~Sinha, K.~A. Sankararaman, A.~Kazerouni, and V.~Avadhanula.
\newblock Multi-{{Armed Bandits}} with {{Cost Subsidy}}.
\newblock In \emph{Proceedings of {{The}} 24th {{International Conference}} on
  {{Artificial Intelligence}} and {{Statistics}}}, pages 3016--3024. PMLR,
  2021.

\bibitem[Sobel(1953)]{sobel1953essentially}
M.~Sobel.
\newblock An essentially complete class of decision functions for certain
  standard sequential problems.
\newblock \emph{The Annals of Mathematical Statistics}, pages 319--337, 1953.

\bibitem[Thompson(1933)]{thompson1933likelihood}
W.~R. Thompson.
\newblock On the likelihood that one unknown probability exceeds another in
  view of the evidence of two samples.
\newblock \emph{Biometrika}, 25\penalty0 (3/4):\penalty0 285--294, 1933.

\bibitem[Wald(1945)]{SPRT_Wald1945}
A.~Wald.
\newblock Sequential {{Tests}} of {{Statistical Hypotheses}}.
\newblock \emph{The Annals of Mathematical Statistics}, 16\penalty0
  (2):\penalty0 117--186, 1945.

\bibitem[Wald and Wolfowitz(1950)]{BayesSolutions_Wald1950}
A.~Wald and J.~Wolfowitz.
\newblock Bayes {{Solutions}} of {{Sequential Decision Problems}}.
\newblock \emph{The Annals of Mathematical Statistics}, 21\penalty0
  (1):\penalty0 82--99, 1950.

\bibitem[Wang et~al.(2023)Wang, Wu, Chen, and Lui]{wang2023_MF}
X.~Wang, Q.~Wu, W.~Chen, and J.~C.~S. Lui.
\newblock Multi-{{Fidelity Multi-Armed Bandits Revisited}}.
\newblock In \emph{Advances in Neural Information Processing Systems},
  volume~36, pages 31570--31600. Curran Associates, Inc., 2023.

\bibitem[Xia et~al.(2016)Xia, Qin, Yu, and Liu]{CostRatioBAI_Xia}
Y.~Xia, T.~Qin, N.~Yu, and T.-Y. Liu.
\newblock Best action selection in a stochastic environment.
\newblock In \emph{Proceedings of the 2016 International Conference on
  Autonomous Agents \& Multiagent Systems}, page 758–766. International
  Foundation for Autonomous Agents and Multiagent Systems, 2016.

\bibitem[Yang et~al.(2024)Yang, Tan, and Jin]{Yang2024BestArm}
J.~Yang, V.~Y.~F. Tan, and T.~Jin.
\newblock Best {{Arm Identification}} with {{Minimal Regret}}.
\newblock \emph{arXiv preprint arXiv:2409.18909}, 2024.

\bibitem[Zhang et~al.(2024)Zhang, Jain, Guo, Chen, Zhou, Suresh, Wagenmaker,
  Sievert, Rogers, Jamieson, et~al.]{zhang2024humor}
J.~Zhang, L.~Jain, Y.~Guo, J.~Chen, K.~Zhou, S.~Suresh, A.~Wagenmaker,
  S.~Sievert, T.~T. Rogers, K.~G. Jamieson, et~al.
\newblock Humor in ai: Massive scale crowd-sourced preferences and benchmarks
  for cartoon captioning.
\newblock In \emph{Advances in Neural Information Processing Systems},
  volume~37, pages 125264--125286. Curran Associates, Inc., 2024.

\bibitem[Zhao et~al.(2023)Zhao, Stephens, Szepesvari, and Jun]{Zhao2023_SR}
Y.~Zhao, C.~Stephens, C.~Szepesvari, and K.-S. Jun.
\newblock Revisiting {{Simple Regret}}: {{Fast Rates}} for {{Returning}} a
  {{Good Arm}}.
\newblock In \emph{Proceedings of the 40th {{International Conference}} on
  {{Machine Learning}}}, pages 42110--42158. PMLR, 2023.

\end{thebibliography}
